%% file: paper.tex
\documentclass[11pt]{article}
\usepackage[margin=1in]{geometry}
\usepackage{amssymb,amsmath,amsthm,amsfonts}
\usepackage{amsfonts}
\usepackage{graphicx}
\usepackage{color-edits}
\usepackage{algorithm}
\usepackage[noend]{algpseudocode}
\usepackage{caption}
\usepackage{natbib}
\usepackage{xspace}
\usepackage{mathtools}
\addauthor{jm}{magenta}
\addauthor{ar}{red}
\addauthor{mj}{blue}

\newtheorem{lemma}{Lemma}
\newtheorem{theorem}{Theorem}

\newtheorem{corollary}{Corollary}

\newcommand{\app}[1]{Appendix~\ref{#1}}
\newcommand{\cX}{\mathcal{X}\xspace}
\newcommand{\cA}{\mathcal{A}\xspace}
\newcommand{\cB}{\mathcal{B}\xspace}

\newcommand{\R}{\ensuremath{\mathbb{R}}\xspace}
\newcommand{\C}{\ensuremath{C}\xspace}
\newcommand{\A}{\ensuremath{\cA}\xspace}
\newcommand{\D}{\ensuremath{\mathcal{D}}\xspace}
\newcommand{\B}{\ensuremath{\cB}\xspace}
\newcommand{\regret}[1]{\textrm{Regret}(#1)}
\newcommand{\reg}[1]{R(#1)}
\newcommand{\pp}[2]{\pi^{#1}_{#2}}
\newcommand{\ppc}[3]{\pp{#1}{#2|#3}}
\newcommand{\rew}[2]{r^{#1}_{#2}}
\newcommand{\Rew}[2]{\D^{#1}_{#2}}
\newcommand{\bt}{\ensuremath{\bot}\xspace}
\newcommand{\ist}{\ensuremath{i_{*}^t}}

\newcommand{\mean}[2]{\ensuremath{\hat{\mu}_{#2}^{#1}}}
\newcommand{\up}[2]{\ensuremath{u_{#2}^{#1}}}
\newcommand{\low}[2]{\ensuremath{\ell_{#2}^{#1}}}
\newcommand{\num}[2]{\ensuremath{n_{#2}^{#1}}}
\newcommand{\interval}[2]{\ensuremath{[\ell_{#2}^{#1}, u_{#2}^{#1}]}}
\newcommand{\unfair}{\textrm{unfair}}
\newcommand{\fairbandits}{\textsc{FairBandits}\xspace}
\newcommand{\kwikfair}{\textsc{KWIKToFair}\xspace}
\newcommand{\fairkwik}{\textsc{FairToKWIK}\xspace}

\theoremstyle{definition}
\newtheorem{definition}{Definition}
\theoremstyle{remark}
\newtheorem{remark}{Remark}

\DeclareMathOperator*{\argmax}{arg\,max}
\newcommand{\E}[1]{\mathbb{E}\left[ #1 \right]}
\newcommand{\Ex}[2]{\mathbb{E}_{#1}\left[ #2 \right]}
\newcommand{\I}[1]{\mathbb{I}\left[ #1 \right]}
\newcommand{\pr}[1]{\mathbb{P}\left[ #1 \right]}
\renewcommand{\Pr}[2]{\mathbb{P}_{#1}\left[ #2 \right]}

\title{Fairness in Learning: Classic and Contextual Bandits\footnote{A condensed version of this work appears in the 30th Annual Conference on Neural Information Processing Systems (NIPS), 2016.}}
\author{Matthew Joseph \and Michael Kearns \and Jamie Morgenstern \and Aaron Roth\thanks{Department of Computer and Information Sciences, University of Pennsylvania. \texttt{\{majos,mkearns,jamiemor,aaroth\}@cis.upenn.edu}. AR is supported in part by an NSF CAREER award, a Sloan Foundation Fellowship, and a Google Faculty Research Award.}}
\begin{document}

\maketitle

\begin{abstract}
  We introduce the study of fairness in multi-armed bandit
  problems. Our fairness definition can be interpreted as demanding
  that given a pool of applicants (say, for college admission or
  mortgages), a worse applicant is never favored over a better one,
  despite a learning algorithm's uncertainty over the true payoffs. We
  prove results of two types:
  
  First, in the important special case of the classic stochastic
  bandits problem (i.e. in which there are no contexts), we provide a
  provably fair algorithm based on “chained” confidence intervals, and
  prove a cumulative regret bound with a cubic dependence on the
  number of arms. We further show that any fair algorithm must have
  such a dependence. When combined with regret bounds for standard
  non-fair algorithms such as UCB, this proves a strong separation
  between fair and unfair learning, which extends to the general
  contextual case.
  
  In the general contextual case, we prove a tight connection between fairness
  and the KWIK (Knows What It Knows) learning model: a KWIK algorithm
  for a class of functions can be transformed into a provably fair
  contextual bandit algorithm, and conversely any fair contextual
  bandit algorithm can be transformed into a KWIK learning
  algorithm. This tight connection allows us to provide a provably
  fair algorithm for the linear contextual bandit problem with a
  polynomial dependence on the dimension, and to show (for a different
  class of functions) a worst-case exponential gap in regret between
  fair and non-fair learning algorithms.\end{abstract}

\newpage
\tableofcontents
\newpage

\input{intro}
\input{prelims}

\input{noncontextual}

\input{noncontextual-lower}
\input{kwiktofair}
\input{fairtokwik}
\input{conjunction}

{\footnotesize \bibliographystyle{plainnat}
\bibliography{fair_bandits}}
\input{appendix}
\end{document}

%% file: intro.tex
\section{Introduction}

Automated techniques from statistics and machine learning are
increasingly being used to make decisions that have important
consequences on people's lives, including hiring~\citep{hiring},
lending \citep{lending}, policing~\citep{policing}, and even criminal
sentencing~\citep{sentencing}. These high stakes uses of machine
learning have led to increasing concern in law and policy circles
about the potential for (often opaque) machine learning techniques to
be \emph{discriminatory} or \emph{unfair}
\citep{cary16,solon16}. Moreover, these concerns are not merely
hypothetical: \citet{Sweeney13} observed that contextual ads for
public record services shown in response to Google searches for
stereotypically African American names were more likely to contain
text referring to arrest records, compared to comparable ads shown in
response to searches for stereotypically Caucasian names, which showed
more neutral text. She confirmed that this was not because of stated
preferences of the advertisers, but rather the automated outcome of
Google's targeting algorithms. Despite the recognized importance of
this problem, very little is known about technical solutions to the
problem of ``unfairness'', or the extent to which ``fairness'' is in
conflict with the goals of learning.\footnote{ For example, a 2014
  White House report \citep{bigdata1} notes that ``[t]he increasing
  use of algorithms to make eligibility decisions must be carefully
  monitored for potential discriminatory outcomes for disadvantaged
  groups, even absent discriminatory intent$\ldots$ additional
  research in measuring adverse outcomes due to the use of scores or
  algorithms is needed to understand the impacts these tools are
  having and will have in both the private and public sector as their
  use grows.''  Along the same lines, a 2016 White House
  report~\citep{bigdata3} observes that ``[a]s improvements in the
  uses of big data and machine learning continue, it will remain
  important not to place too much reliance on these new systems
  without questioning and continuously testing the inputs and
  mechanics behind them and the results they produce."  Similarly, in
  a recent speech FTC Commissioner Julie Brill~\citep{bigdata2}
  observed, ``$\ldots$ a lot remains unknown about how big data-driven
  decisions may or may not use factors that are proxies for race, sex,
  or other traits that U.S. laws generally prohibit from being used in
  a wide range of commercial decisions $\ldots$ What can be done to
  make sure these products and services--–and the companies that use
  them – treat consumers fairly and ethically?''}

In this paper, we consider the extent to which a natural fairness
notion is compatible with learning in a general setting (the
\emph{contextual bandit setting}), which can be used to model many of
the applications mentioned above in which machine learning is currently
employed. In this model, the learner is a sequential decision maker,
which must choose at each time step $t$ which decision to make, out of
a finite set of $k$ choices (for example, which of $k$ loan applicants
-- potentially from different populations or racial groups -- to give
a loan to). Before the learner makes its decision at round $t$, it
observes some \emph{context} $x_j^t$ for each choice of arm $j$
($x_j^t$ could, for example, represent the contents of the loan
application of an individual from population $j$ at round $t$). When
the learner chooses arm $j$ at time $t$, it obtains a stochastic
\emph{reward} $r_j^t$ whose expectation is determined by some unknown
function of the context: $\E{r_j^t} = f_j(x_j^t)$. The goal of the
learning algorithm is to maximize its expected reward -- i.e. to
approximate the optimal policy, which at each round, chooses arm
$j$ to maximize $\E{r_j^t}$.
The difficulty in this task stems from
the unknown functions $f_j$ which map contexts to rewards; these
functions must be learned. Despite this, there are many known
algorithms for learning the optimal policy (in the absence of any
fairness constraint).

\subsection{Fairness and Learning}
Our notion of individual fairness is very simple: it states that it is
\emph{unfair} to preferentially choose one individual (e.g. for a
loan, a job, admission to college, etc.) over another if he or she is
not as qualified as the other individual. This definition of fairness
is apt for our setting, since in contextual learning, the quality of
an arm is clear: its expected reward. We view different arms $j$ as
representing different populations (e.g. different ethnic groups,
cultures, or other divisions within society), and view the context
$x_j^t$ at round $t$ as representing information about a particular
individual from that population. Each population has its own
underlying function $f_j$ which maps contexts to expected
payoff\footnote{It is natural that different populations should have
  different underlying functions -- for example, in a college
  admissions setting, the function mapping applications to college
  success probability might weight SAT scores less in a wealthy
  population that employs SAT tutors, and more in a working-class
  population that does not -- see \citet{dwork2012fairness} for more
  discussion of this issue and \citet{bigdata3} for examples.}. At
each time step $t$, the algorithm is asked to choose between specific
members of each population, represented by the contexts $x_j^t$. The
quality of an individual is thus exactly $\E{r_j^t} = f_j(x_j^t)$.
Our fairness condition translates thus: for any pair of arms $j, j'$
at time $t$, if $f_j(x_j^t) \geq f_{j'}(x_{j'}^t)$, then an algorithm
is said to be discriminatory if it preferentially chooses the lower
quality arm $j'$.  Said another way, an algorithm is \emph{fair} if it
guarantees the following: with high probability, over all rounds $t$,
and for all pairs of arms $j, j'$, whenever
$f_j(x_j^t) \geq f_{j'}(x_{j'}^t)$, the algorithm chooses arm $j$ with
probability at least that with which it chooses arm $j'$\footnote{Note
  that the definition does not require \emph{equality of outcomes} on
  a population wide basis, also known as \emph{statistical parity}. If
  some population $j$ is less credit-worthy on average than another
  population $j'$, we do not necessarily say that an algorithm is
  discriminatory if it ends up giving fewer loans to individuals from
  population $j$.  Our notion of discrimination is on an individual
  basis -- it requires that even if population $j$ is less credit
  worthy on average than population $j'$, if it happens that on some
  day, an individual appears from population $j$ who is at least as
  credit worthy as the individual from population $j'$, then the
  algorithm cannot favor the individual from population $j'$.}.

It is worth noting that this definition of fairness (formalized in the
preliminaries) is entirely consistent with the optimal policy, which
can simply choose at each round to play uniformly at random from the
arms $\argmax_{j}\left(\E{r_j^t}\right)$ which maximize the
expected reward. This is because -- it seems -- the goal of fairness
as enunciated above is entirely consistent with the goal of maximizing
expected reward. Indeed, the fairness constraint exactly states that
the algorithm \emph{cannot} favor low reward arms!

Our main conceptual result is that this intuition is incorrect in the
face of unknown reward functions. Even though the constraint of
fairness is consistent with \emph{implementing} the optimal policy, it
is not necessarily consistent with \emph{learning} the optimal
policy. We show that fairness always has a cost, in terms of the
achievable learning rate of the algorithm. For some problems, the cost
is mild, but for others, the cost is large.

\subsection{Our Results}
We divide our results into two parts. First, we study the classic
stochastic multi-armed bandit problem
\citep{LR85,katehakis1995sequential}. In this case, there are no
contexts, and each arm $i$ has a fixed but unknown average reward
$\mu_i$. Note that this is a special case of the contextual bandit
problem in which the contexts are the same every day. In this setting,
our fairness constraint specializes to require that with probability
$1-\delta$, for any pair of arms $i,j$ for which
$\mu_i \geq \mu_j$, at no round $t$ does the algorithm play arm $j$
with probability higher than that with which it plays arm $i$. Note
that even this special case models interesting scenarios from the
point of view of fairness in learning. It models, for example, the
case in which choices are made by a loan officer after applicants have
been categorized into $k$ internally indistinguishable equivalence
classes based on their applications.

Without a fairness constraint, it is known that it is possible to
guarantee non-trivial regret to the optimal policy after only
$T = O(k)$ many rounds~\citep{ACF02}.  In Section
\ref{sec:fairbandits}, we give an algorithm that satisfies our
fairness constraint and is able to guarantee non-trivial regret after
$T = O(k^3)$ rounds. We then show in Section \ref{sec:lower} that it
is not possible to do better -- \emph{any} fair learning algorithm can
be forced to endure constant per-round regret for $T = \Omega(k^3)$
rounds. Thus, we tightly characterize the optimal regret attainable by
fair algorithms in this setting, and formally separate it from the
regret attainable by algorithms absent a fairness constraint.  Note
that this already shows a separation between the best possible
learning rates for contextual bandit learning with and without the
fairness constraint -- the stochastic multi-armed bandit problem is a
special case of every contextual bandit problem, and for general
contextual bandit problems, it is also known how to get non-trivial
regret after only $T = O(k)$ many rounds
~\citep{AHKLLS14,BLLRS11,CLRS11}.

We then move on to the general contextual bandit setting and prove a
broad characterization result, relating fair contextual bandit
learning to \emph{KWIK} learning~\citep{li2011knows}. The KWIK model,
which stands for \emph{Knows What it Knows} and has a close
relationship with reinforcement learning, is a model of sequential
supervised classification in which the learning algorithm must be
confident in its predictions. Informally, a KWIK learning algorithm
receives a sequence of unlabeled examples, whose true labels are
defined by some unknown function in a class $C$.  For each example,
the algorithm may either predict a label, or announce ``I Don't
Know''. The KWIK requirement is that with high probability, for each
example, if the algorithm predicts a label, then its prediction must
be very close to the true label. The quality of a KWIK learning
algorithm is characterized by its ``KWIK bound'', which provides an
upper bound on the maximum number of times the algorithm can be forced
to announce ``I Don't Know''.  For any contextual bandit problem
(defined by the set of functions $C$ from which the payoff functions
$f_j$ may be selected), we show that the optimal learning rate of any
fair algorithm is determined by the best KWIK bound for the class
$C$. We prove this constructively -- we give a reduction showing how
to convert a KWIK learning algorithm into a fair contextual bandit
algorithm in Section \ref{sec:kwiktofair}, and vice versa in Section
\ref{sec:fairtokwik}. Both reductions show that the KWIK bound of the
KWIK algorithm is polynomially related to the regret of the fair
algorithm.

This general connection has %several
immediate implications, because it allows us to import known results
for KWIK learning~\citep{li2011knows}. For example, it implies that
some fair contextual bandit problems are \emph{easy}, in that there
are fair algorithms which can obtain non-trivial regret guarantees
after %only
polynomially many rounds. This is the case, for example, for the
important linear special case in which the contexts
$x_j^t \in \mathbb{R}^d$ are real valued vectors and the unknown
functions $f_j$ are linear:
$f_j(x_j^t) = \langle \theta^{j}, x_j^t \rangle$\footnote{This
  corresponds to the case in which the probability that an individual
  pays back his or her loan is determined by a standard linear
  regression model.}.  In this case, the KWIK-learnability of noisy
linear regression problems~\citep{strehl2008online,li2011knows}
implies that we can construct a fair contextual bandit algorithm whose
per-round regret is polynomial in $d$. Conversely, it also implies
that some contextual bandit problems which are easy without the
fairness constraint become \emph{hard} once we impose the fairness
constraint, in that any fair algorithm must suffer constant per-round
regret for exponentially many rounds. This is the case, for example,
when the context consists of boolean vectors $x_j^t \in \{0,1\}^d$,
and the unknown functions $f_j:\{0,1\}^d \rightarrow \{0,1\}$ are
\emph{conjunctions} -- the ``and''s of some unknown set of
features\footnote{For example, a conjunction might predict that an
  individual is likely to pay back his loan if all of the following
  conditions are satisfied: he or she has graduated from college, has
  a clean driving history, and has not previously defaulted on any
  loans.}.  The impossibility of non-trivial KWIK-learning of
conjunctions~\citep{li09,li2011knows} implies that no fair learner in
the contextual bandit setting can achieve non-trivial regret before
exponentially many (in $d$) rounds.

\subsection{Other Related Work}
Several papers study the problem of fairness in machine learning. One
line of work aims to give algorithms for batch classification which
achieve \emph{group fairness} otherwise known as \emph{equality of
  outcomes}, \emph{statistical parity} -- or algorithms that avoid
\emph{disparate impact} (see
e.g.~\citet{CV10,LRT11,KAS11,FFMSV15,FKL16} and \cite{AFFRSSV16} for a
study of \emph{auditing} existing algorithms for disparate
impact). While statistical parity is sometimes a desirable goal --
indeed, it is sometimes required by law -- as observed
by~\citet{dwork2012fairness} and others, it suffers from two
problems. First, if different populations indeed have different
statistical properties, then it can be at odds with accurate
classification. Second, even in cases when statistical parity is
attainable with an optimal classifier, it does not prevent
discrimination at an individual level -- see \cite{dwork2012fairness}
for a catalog of ways in which statistical parity can be insufficient
from the perspective of fairness.  In contrast, we study a notion
aimed at guaranteeing fairness at the individual level.

Our definition of fairness is most closely related to that
of~\citet{dwork2012fairness}, who proposed and explored the basic
properties of a technical definition of individual fairness
formalizing the idea that ``similar individuals should be treated
similarly''. Specifically, their work presupposes the existence of a
task-specific metric on individuals, and proposes that fair algorithms
should satisfy a Lipschitz condition with respect to this metric. Our
definition of fairness is similar, in that the expected reward of each
arm is a natural metric through which we define fairness. The main
conceptual distinction between our work and~\citet{dwork2012fairness}
is that their work operates under the assumption that the metric is
known to the algorithm designer, and hence in their setting, the
fairness constraint binds only insofar as it is in conflict with the
desired outcome of the algorithm designer. The most challenging aspect
of this approach (as they acknowledge) is that it requires that some
third party design a ``fair'' metric on individuals, which in a sense
encodes much of the relevant challenge. The question of how to design
such a metric was considered by~\citet{Zem13}, who study methods to
learn representations that encode the data, while obscuring protected
attributes. Our fairness constraint, conversely, is entirely aligned
with the goal of the algorithm designer in that it is satisfied by the
optimal policy; nevertheless, it affects the space of feasible
learning algorithms, because it interferes with \emph{learning} an
optimal policy, which depends on the unknown reward functions.

At a technical level, our work is related to~\citet{amin2012graphical}
and~\citet{amin2013large}, which also relate KWIK learning to bandit
learning in a different context, unrelated to fairness (when the arm
space is very large).

%% file: prelims.tex
\newcommand{\is}{i^{*}}
\section{Preliminaries}\label{sec:prelims}
We study the \emph{contextual bandit} setting, which is
defined by a domain $\cX$, a set of ``arms'' $[k] := \{1,\ldots,k\}$
and a class $\C$ of functions of the form $f:\cX\rightarrow [0,1]$.
For each arm $j$ there is some function $f_j \in C$, unknown to the
learner. In rounds $t = 1, \ldots, T$, an adversary reveals to the
algorithm a \emph{context} $x_j^t$ for each arm\footnote{Often, the
  contextual bandit problem is defined such that there is a single
  context $x^t$ every day. Our model is equivalent -- we could take
  $x_j^t := x^t$ for each $j$.}.  An algorithm $\cA$ then chooses an
arm $i_t$, and observes stochastic reward $\rew{t}{i_t}$ for the arm
it chose. We assume $\rew{t}{j}\sim \Rew{t}{j}$,
$\E{\rew{t}{j}} = f_j(x^t_j)$, for some distribution $\Rew{t}{j}$
over $[0,1]$.

Let $\Pi$ be the set of policies mapping contexts to distributions
over arms $X^k \to \Delta^k$, and $\pi^*$ the optimal policy which
selects a distribution over arms as a function of contexts to maximize
the expected reward of those arms.  The {\bf pseudo-regret} of an
algorithm $\A$ on contexts $x^1, \ldots, x^T$ is defined as follows,
where $\pp{t}{}$ represents $\A$'s distribution on arms at round $t$:
\[ \sum_{t}\Ex{i^t_* \sim \pi^*(x^t)}{f_{i^t_*}(x^t_{i^t_*})}  -  \Ex{i^t \sim \pp{t}{}}{\sum_t f_{i^t}(x^t_{i^t})}  = \regret{x^1, \ldots, x^T}.\]
We hereafter refer to this as the {\bf regret} of $\A$.  The optimal
policy $\pi^*$ pulls arms with highest expectation at each round, so:
\[\regret{x^1, \ldots, x^T} = \sum_{t}\max_j\left(f_j(x_j^t)\right) - \Ex{i^t \sim \pp{t}{}}{\sum_t f_{i^t}(x^t_{i^t})}. \]
We say that $\A$ satisfies regret bound $\reg{T}$ if
$\max_{x^1, \ldots, x^T}\regret{x^1, \ldots, x^t} \leq \reg{T}$.

Let the history
$h^t \in \left(\cX^{k} \times [k]\times [0,1] \right)^{t-1}$ be a
record of $t-1$ rounds experienced by $\A$, $t-1$ 3-tuples which
encode for each $t$ the realization of the contexts, arm chosen, and
reward observed.
%In particular, for any $t' < t$, the history
%$h^{t, t'} = (x^{t'}, i^{t'}, y^{t'})$ contains $x^{t'}$, the contexts
%in round $t' < t$, $i^{t'}\sim \ppc{t'}{}{h^{t'}}$, the arm chosen at
%time $t' < t$ according to $\A$'s distribution over arms facing
%contexts $x^{t'}$, conditioned on the history $h^{t'}$ observed by the
%algorithm, and $y^t \sim \Rew{t}{i^t}$ the realized payoff of the
%selected arm in round $t$.
We write $\ppc{t}{j}{h^t}$ to denote the probability that $\A$ chooses
arm $j$ after observing contexts $x^t$, given $h^{t}$.  For notational
simplicity, we will often drop the superscript $t$ on the history when
referring to the distribution over arms:
$\ppc{t}{j}{h} \coloneqq \ppc{t}{j}{h^t}$.

We now define what it means for a contextual bandit algorithm to be
$\delta$-fair with respect to its arms. Informally, this will mean
that $\A$ will play arm $i$ with higher probability than arm $j$ in
round $t$ only if $i$ has higher mean than $j$ in round $t$, for all
$i,j\in [k]$, and in all rounds $t$.

\begin{definition}[$\delta$-fair]\label{def:fair}
  $\A$ is $\delta$-{\bf fair} if, for all sequences of contexts
  $x^1, \ldots, x^{t}$ and all payoff distributions
  $\Rew{t}{1}, \ldots, \Rew{t}{k}$, with probability at least
  $1-\delta$ over the realization of the history $h$, for all rounds
  $t\in [T]$ and all pairs of arms $j, j' \in [k]$,
  $$\ppc{t}{j}{h} > \ppc{t}{j'}{h}\ \textrm{only if}\
  f_j(x^t_j) > f_{j'}(x^t_{j'}).$$
\end{definition}

\begin{remark}
  Definition~\ref{def:fair} prohibits favoring lower payoff arms over
  higher payoff arms. One relaxed definition only requires that
  $\ppc{t}{j}{h} = \ppc{t}{j'}{h}$ when
  $f_j(x^t_j) = f_{j'}(x^t_{j'})$ -- requiring only \emph{identical}
  individuals (concerning expected payoff) be treated
  identically. This relaxation is a special case
  of~\citet{dwork2012fairness}'s proposed family of definitions, which
  require that ``similar individuals be treated similarly''. We use
  Definition~\ref{def:fair} as it is better motivated in its
  implications for fair treatment of individuals, but all of our
  results -- including our lower bounds -- apply also to this
  relaxation.
\end{remark}
\paragraph{KWIK learning}
%We give the definition of KWIK learning a class $\C : \cX \to
%[0,1]$.
Let $\B$ be an algorithm which takes as input a sequence of examples
$x^1,\ldots,x^T$, and when given some $x^t\in \cX$, outputs either a
prediction $\hat{y}^t \in [0,1]$ or else outputs $\hat{y}^t = \bot$,
representing ``I don't know''. When $\hat{y}^t = \bot$, $\B$ receives
feedback $y^t$ such that $\E{y^t} = f(x^t)$.  $\B$ is an
$(\epsilon, \delta)$-KWIK learning algorithm for $\C : \cX \to [0,1]$, with KWIK bound
$m(\epsilon, \delta)$ if for any sequence of examples
$x^1 , x^2, \ldots$ and any target $f\in \C$, with probability at
least $1-\delta$, both:
\begin{enumerate}
\item Its numerical predictions are accurate: for all $t$,
  $\hat{y}^t \in \{\bt\} \cup [f(x^t)-\epsilon, f(x^t)+\epsilon]$, and
\item $\B$ rarely outputs ``I Don't Know'':
  $\sum_{t= 1}^\infty \I{\hat{y}^t = \bot} \leq m(\epsilon, \delta)$.
\end{enumerate}

\subsection{Specializing to Classic Stochastic Bandits}
In Sections \ref{sec:fairbandits} and \ref{sec:lower}, we study the
classic stochastic bandit problem, an important special case of the
contextual bandit setting described above. Here we specialize our
notation to this setting, in which there are no contexts.  For each
arm $j \in [k]$, there is an unknown distribution $\Rew{}{j}$ over
$[0,1]$ with unknown mean $\mu_j$. A learning algorithm $\cA$ chooses
an arm $i_t$ in round $t$, and observes the reward
$\rew{t}{i_t}\sim \Rew{}{i_t}$ for the arm that it chose.  Let
$\is\in [k]$ be the arm with highest expected reward:
$\is \in \argmax_{i\in [k]} \mu_i$.  The pseudo-regret of an algorithm
$\A$ on $\Rew{}{1}, \ldots, \Rew{}{k}$ is now just:
\[  T \cdot \mu_{\is}  -  \Ex{i^t \sim \pp{t}{}}{\sum_{0 \leq t \leq T} \mu_{i^t}}  = \regret{T, \Rew{}{1}, \ldots,\Rew{}{k}}\]
%
%where $\pp{t}{}$ is $\A$'s distribution over arms in the $t$th
%round.   We say $\A$ satisfies regret bound $R(T)$ if
%$\max_{\Rew{}{1}, \ldots,\Rew{}{k}}\regret{T, \Rew{}{1},
%  \ldots,\Rew{}{k}} \leq R(T)$.
%We now define the \emph{history} of an algorithm $\A$ prior to round
%$t$ and the probability distribution that the algorithm uses to choose
%arms at time $t$ conditioned on its history.
Let
$h^t \in \left([k]\times [0,1] \right)^{t-1}$ denote a record of the
$t-1$ rounds experienced by the algorithm so far, represented by $t-1$
2-tuples encoding the previous arms chosen and rewards observed.
We write $\ppc{t}{j}{h^t}$ to denote the probability that $\A$ chooses
arm $j$ given history $h^{t}$. Again, we will
often drop the superscript $t$ on the history when referring to the
distribution over arms: $\ppc{t}{j}{h} \coloneqq \ppc{t}{j}{h^t}$.

$\delta$-fairness in the classic bandit setting specializes as
follows:
\begin{definition}[$\delta$-fairness in the classic bandits setting]
  $\A$ is $\delta$-{\bf fair} if, for all distributions
  $\Rew{}{1}, \ldots, \Rew{}{k}$, with probability at least $1-\delta$
  over the history $h$, for all $t\in [T]$ and all $j, j' \in [k]$:
  $$\ppc{t}{j}{h} > \ppc{t}{j'}{h}\ \textrm{only if}\  \mu_j > \mu_{j'}.$$
\end{definition}

%% file: noncontextual.tex
\section{Fair Classic Stochastic Bandits: An Algorithm}\label{sec:fairbandits}

In this section, we describe a simple and intuitive modification of
the standard UCB algorithm~\citep{ACF02}, called \fairbandits,
prove that it is fair, and analyze its regret bound. The algorithm and
its analysis highlight a key idea that is important to the design of
fair algorithms in this setting: that of \emph{chaining} confidence
intervals. Intuitively, as a $\delta$-fair algorithm explores
different arms it must play two arms $j_1$ and $j_2$ with equal
probability until it has sufficient data to deduce, with confidence
$1-\delta$, either that $\mu_{j_1} > \mu_{j_2}$ or vice
versa. \fairbandits does this by maintaining empirical
estimates of the means of both arms, together with confidence
intervals around those means. To be safe, the algorithm must play the
arms with equal probability while their confidence intervals
overlap. The same reasoning applies simultaneously to every pair of
arms. Thus, if the confidence intervals of each pair of arms $j_i$ and
$j_{i+1}$ overlap for each $i \in [k]$, the algorithm is forced to
play \emph{all} arms $j$ with equal probability. This is the case even
if the confidence intervals around arm $j_k$ and arm $j_1$ are far
from overlapping -- i.e. when the algorithm can be confident that
$\mu_{j_1} > \mu_{j_k}$.

This approach initially seems naive: in an attempt to achieve
fairness, it seems overly conservative when ruling out arms, and can
be forced to play arms uniformly at random for long periods of
time. This is reflected in its regret bound, which is only non-trivial
after $T \gg k^3$, whereas the UCB algorithm~\citep{ACF02} achieves
non-trivial regret after $T = O(k)$ rounds.  However, our lower bound
in Section~\ref{sec:lower} shows that \emph{any} fair algorithm
\emph{must} suffer constant per-round regret for $T \gg k^3$ rounds on
some instances.
% Thus, we show a strict (polynomial) separation between the learning
% rates that are possible for fair and non-fair no-regret algorithms
% for multi-armed bandit algorithms.  Note that because this setting
% is a special case of every instance of the more general contextual
% bandit problem (studied in our companion paper \cite{nips2}), our
% $\Omega(k^3)$ lower bound applies to the general setting as well,
% and thus also implies a strict separation between fair and non-fair
% contextual bandit algorithms.

%and say $i$ is \emph{chained} to $j$ if
%
%At each round $t$, define a graph $G^t = (V^t, E^t)$ with respect to
%the active set $S$ as follows. The vertices of the graph correspond to
%the active set of arms at round $t-1$: $V^t = S^{t-1}$.  There is an
%edge $(i,j) \in E^t$ if and only if the confidence intervals of arms
%$i$ and $j$ overlap at time $t$:
%$[\low{t}{i}, \up{t}{i} ]\cap [\low{t}{j}, \up{t}{j}] \neq \emptyset$.
%We say that arms $i,j \in V^t$ are \emph{chained} in $S^{t-1}$ if $i$
%and $j$ are part of the same connected component in $G^t$.
%

We now give an overview of the behavior of \fairbandits. At every
round $t$, \fairbandits identifies the arm $\ist = \argmax_i\up{t}{i}$
that has the largest \emph{upper} confidence interval amongst the
active arms. At each round $t$, we say $i$ is \emph{linked} to $j$ if
$[\low{t}{i}, \up{t}{i} ]\cap [\low{t}{j}, \up{t}{j}] \neq \emptyset$,
and $i$ is \emph{chained} to $j$ if $i$ and $j$ are in the same
component of the transitive closure of the linked
relation. \fairbandits plays uniformly at random among all active arms
chained to arm $\ist$.
%It obtains
%reward $\rew{t}{i}$ for the arm $i$ that it played, and then updates
%the confidence interval of arm $i$.

Initially, the active set contains all arms. The active set of arms at
each subsequent round is defined to be the set of arms that are
chained to the arm with highest upper confidence bound at the previous
round. The algorithm can be confident that arms that have become
unchained to the arm with the highest upper confidence bound at any
round have means that are lower than the means of any chained arms,
and hence such arms can be safely removed from the active set, never
to be played again. This has the useful property that the active set
of arms can only shrink: at any round $t$,
$S_t \subseteq S_{t-1}$; see Figure \ref{fig:CIs} for an example of
active set evolution over time.

\begin{center}
    \begin{algorithm}
	\begin{algorithmic}[1]
	\Procedure{\fairbandits}{$\delta$}
 	\State  $S^0 \gets \{1,\ldots,k\}$ \Comment{Initialize the active set}
 	\For{$i = 1, \ldots k$}
 		\State $\mean{0}{i} \gets \frac{1}{2}$,  $\up{0}{i} \gets 1$, $\low{0}{i} \leftarrow 0$, $\num{0}{i} \gets 0$  \Comment{Initialize each arm}
 	\EndFor
 	\For{$t =1$ to $T$}
 		\State $\ist \gets \arg \max_{i\in S^{t-1}} \up{t}{i}$ \Comment{Find arm with highest ucb}
 		\State $S^t \gets \{j \mid $ $j$ chains to $\ist$, $j \in S^{t-1}\}$ \Comment{Update active set}
 		\State $j^* \gets (x \in_R S^t)$ \Comment{Select active arm at random}
 		\State $\num{t+1}{j^*} \gets \num{t}{j^*} + 1$
 		\State $\mean{t+1}{j^*} \gets \frac{1}{\num{t+1}{j^*}} (\mean{t}{j^*}\cdot \num{t}{j^*} + \rew{t}{j^*})$ \Comment{Pull arm $j^*$, update its mean estimate}
                \State $B \gets \sqrt{\frac{\ln((\pi\cdot (t+1))^2/3\delta)}{2\num{t+1}{j^*}}}$
  		\State $\left[\low{t+1}{j^*}, \up{t+1}{j^*}\right] \gets  \left[ \mean{t+1}{j^*}  - B, \mean{t+1}{j^*} + B \right]$
 \Comment{Update interval for pulled arm}
 		\For{$j \in S^t, j \neq j^*$}
 			\State $\mean{t+1}{j} \gets  \mean{t}{j}$, $\num{t+1}{j} \gets \num{t}{j}$, $\up{t+1}{j} \gets \up{t}{j}, \low{t+1}{j} \gets \low{t}{j}$
 		\EndFor
 	\EndFor
 	\EndProcedure
 	\end{algorithmic}\label{alg:ucbfair}
 \end{algorithm}
\end{center}

%Before we prove that \fairbandits is $\delta$-fair, we informally describe why it
%satisfies the fairness definition. All arms will initially be played
%with equal probability, and have confidence intervals that span all of
%$[0,1]$. As the arms are played, we maintain confidence intervals for
%each active arm, which by a upperChernoff bound will contain the arm's true
%mean, for all times $t$ and all active arms, with probability
%$1-\delta$. So, if there is probability greater than $\delta$ of
%$\mu_i \geq \mu_j$, their intervals will overlap or $i$'s interval
%will be entirely above $j$'s interval.  Finally, the algorithm will play with
%equal probability any arm which belongs to the set of arms which are
%``chained'' to the arm with highest upper confidence bound.  Thus, any
%arm which is no longer being played has strictly lower mean than the
%set of arms being played (with probability $1-\delta$).  So, since
%there are only two probabilities of arms being played (either UAR
%amongst the active set or $0$), this suffices to guarantee fairness.
We first observe that with probability $1-\delta$, all of the
confidence intervals maintained by \fairbandits$(\delta)$ contain the
true means of their respective arms over all rounds. We prove this claim, along with all other claims in this section without proofs, in Appendix ~\ref{sec:app-noncontextual-regret}.

\begin{lemma}\label{lem:intervals}
  With probability at least $ 1 - \delta$, for every arm $i$ and
  round $t$ $\low{t}{i} \leq \mu_i \leq \up{t}{i}$.
\end{lemma}

The fairness of \fairbandits follows almost immediately from this
guarantee.

\begin{theorem}\label{thm:noncontextual-fair}
    \fairbandits$(\delta)$ is $\delta$-fair.
\end{theorem}
\begin{proof}
  By Lemma~\ref{lem:intervals}, with probability at least $1-\delta$
  all confidence intervals contain their true means across all rounds.
  Thus, with probability $1-\delta$, at every round $t$, for every
  $i\in S^t, j \notin S^t$, it must be that $\mu_j < \mu_i$ -- the
  arms not in the active set have strictly smaller means than those in
  the active set; if not,
  $\up{t}{j} \geq \mu_j \geq \mu_i \geq \low{t}{i}$ implies $j$ would
  be chained to $\ist$ if $i$ is.  Finally, all arms in $S^t$ are
  played uniformly at random -- but since all such arms are played
  with the same probability, this does not cause the fairness
  constraint to bind for any pair $i, i' \in S^t$, for any realization
  of $\mu_i,\mu_i'$ which lie within their confidence intervals.
\end{proof}

\begin{figure}
  \begin{minipage}[c]{0.6\textwidth}
  \includegraphics[width=9cm]{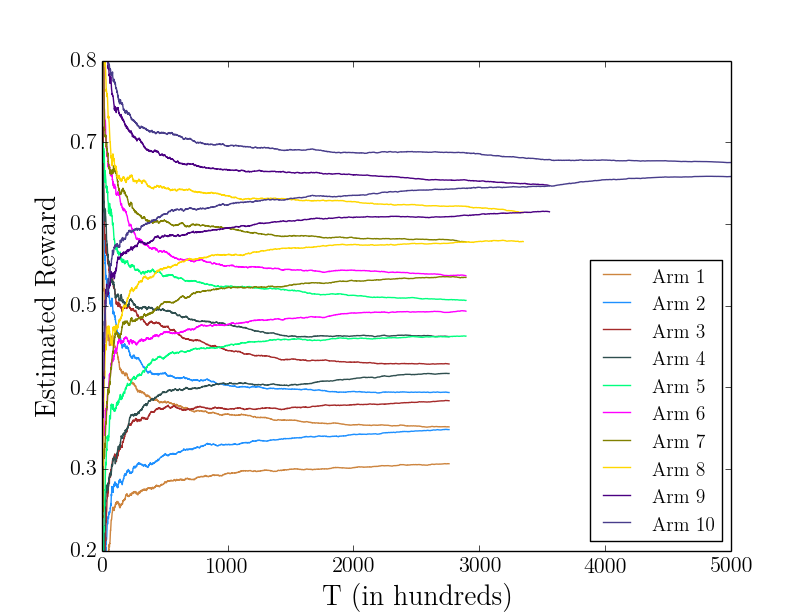}
  \end{minipage}\hfill
  \begin{minipage}[c]{0.35\textwidth}
    \caption{Confidence intervals over time for the lower bound
      instance outlined in Section \ref{sec:lower} for $k = 10$. Lines
      correspond to upper and lower confidence bounds for each arm
      and cut off at the round in which the arm leaves the active set.}
  \label{fig:CIs}
  \end{minipage}
\end{figure}

Next, we upper bound the regret of \fairbandits.
\begin{theorem}\label{thm:noncontextual-regret}
  If $\delta < 1/\sqrt{T}$, then \fairbandits has regret
  $$R(T) = O \left(\sqrt{k^3T\ln\frac{Tk}{\delta}} \right).$$
\end{theorem}

\begin{remark}
  Before proving Theorem \ref{thm:noncontextual-regret}, we highlight
  two points. First, this bound becomes non-trivial (i.e. the average
  per-round regret is $\ll 1$) for $T = \Omega(k^3)$. As we show in
  the next section, it is not possible to improve on this. Second, the
  bound may appear to have suboptimal dependence on $T$ when compared
  to unconstrained regret bounds (where the dependence on $T$ is often
  described as logarithmic). However, it is known that
  $\Omega\left(\sqrt{kT}\right)$ regret is \emph{necessary} even in
  the unrestricted setting (without fairness) if one does not make
  data-specific assumptions on an instance~\citep{bubeck2012regret}
  (e.g. that there is a lower bound on the gap between the best and
  second best arm). It would be possible to state a logarithmic
  dependence on $T$ in our setting as well while making assumptions on
  the gaps between arms, but since our fairness constraint manifests
  itself as a cost that depends on $k$, we choose for clarity to avoid
  such assumptions. Without such assumptions, our dependence on $T$ is
  also optimal.
\end{remark}

We now prove
Theorem~\ref{thm:noncontextual-regret}. Lemma~\ref{lem:num-pulled}
upper bounds the probability any arm $i$ active in round $t$ has been
pulled substantially fewer times than its expectation, i.e.
$\num{t}{i} \ll \frac{t}{k}$. Lemma~\ref{lem:interval-width} upper
bounds the width of any confidence interval used by \fairbandits in
round $t$ by $\eta(t)$, conditioned on $i$ being pulled the number of
times guaranteed by Lemma~\ref{lem:num-pulled}. Finally, we stitch
this together to prove Theorem~\ref{thm:noncontextual-regret} by upper
bounding the total regret incurred for $T$ rounds by noticing that the
regret of any arm active in round $t$ is at most $k \eta(t)$.

We begin by lower bounding the probability that any arm active in
round $t$ has been pulled substantially fewer times than its
expectation.
\begin{lemma}\label{lem:num-pulled}
With probability at least $1-\frac{\delta}{2t^2}$,
\[\num{t}{i} \geq \frac{t}{k} - \sqrt{\frac{t}{2}\ln\left(\frac{2k \cdot t^2}{\delta}\right)}\]
for all $i\in S^t$ (for all active arms in round $t$).
\end{lemma}

We now use this lower bound on the number of pulls of active arm $i$
in round $t$ to upper-bound $\eta(t)$, an upper bound on the
confidence interval width \fairbandits uses for any active arm $i$ in
round $t$.

\begin{lemma}\label{lem:interval-width}
  Consider any round $t$ and any arm $i\in S^t$.  Condition on
  $\num{t}{i} \geq \frac{t}{k} -
  \sqrt{\frac{t\ln(\frac{2k t^2}{\delta})}{2}}$.
Then,
\[\up{t}{i} - \low{t}{i} \leq 2 \sqrt{\frac{\ln\left( \left(\pi \cdot t\right)^2/3\delta\right)}{2 \cdot \frac{t}{k} -
  \sqrt{\frac{t\ln(\frac{2kt^2}{\delta})}{2}}}} = \eta(t).\]
\end{lemma}

Finally, we prove the bound on the total regret of the algorithm,
using the bound on the width of any active arm's confidence interval
in round $t$ provided by Lemma~\ref{lem:interval-width}.

\begin{proof}[Proof of Theorem~\ref{thm:noncontextual-regret}]
  We condition on $\mu_i \in \interval{t}{i}$ for all $i, t$. This occurs
  with probability at least $1-\delta$, by Lemma~\ref{lem:intervals}.
  We claim that this implies that arm $i_*$ with highest expected
  reward is always in the active set. This follows from the fact that
  $\mu_{i_*}\in \interval{t}{i_*}$ and $\mu_{j}\in \interval{t}{j}$ for all
  $j, t$; thus, if $\mu_{i_*} > \mu_j$, it must be that
  $\up{t}{i_*} \geq \low{t}{j}$. Thus, this holds for $\ist$, the arm
  with highest upper confidence bound in round $t$, so $i_*$ must be
  chained to $\ist$ in round $t$ for all $t$.

  We further condition on the event that for all $j, t$,
\[\num{t}{j} \geq \frac{t}{k} - \sqrt{\frac{t}{2}\ln\left(\frac{2k t^2}{\delta}\right)},\]
which holds with probability at least $1-\frac{\pi \delta}{2}$ by
Lemma~\ref{lem:num-pulled} and a union bound over all times $t$.  This
implies that, for all rounds $t$, for every active arm $j\in S^t$,
Lemma~\ref{lem:interval-width} applies, and therefore
\[\up{t}{j} - \low{t}{j} \leq \eta(t). \]
Finally, we upper-bound the per-round regret of pulling any active arm
$i\in S^t$ at round $t$. Since $i_*$ is active, any $i\in S^t$ is
chained to arm $i_*$. Since all active arms have confidence interval
width at most $\eta(t)$ and $i$ must be chained using at most $k$
arms' confidence intervals, we have that
\[\low{t}{i} \geq \up{t}{i_*} - k \cdot \eta(t).\]
Since $\mu_i \geq \low{t}{i}$ and $\up{t}{i_*} \geq \mu_{i_*}$, it
follows that $|\mu_i - \mu_{i_*}| \leq k \cdot \eta(t)$ for any
$i\in S^t$. Finally, summing up over all rounds $t\in T$, we know that
 \begin{align*}
   R(T)& \leq \sum_{t: 0}^T \min(1, k \cdot \eta(t)) + \left(1 + \frac{\pi}{2}\right)\delta T\\
%        &\leq \sum_{t: 0}^T k \cdot \min(1, \eta(t))+ \delta T\\
%        & \leq  k \left(\sum_{t: \frac{t}{k} > 2 \sqrt{t\ln\frac{tk}{\delta}} }^T  \sqrt{\frac{\ln\frac{t}{\delta}}{\frac{t}{k} - \sqrt{t\ln\frac{tk}{\delta}}}} + \sum_{t: \frac{t}{k} \leq 2 \sqrt{t\ln\frac{tk}{\delta}} }^T 1  \right) + \delta T
         & \leq  k \left(\sum_{t: \frac{t}{k} > 2 \sqrt{t\ln\frac{2tk}{\delta}} }^T  \sqrt{\frac{\ln\frac{t}{\delta}}{\frac{t}{2k}}}+ \sum_{t: \frac{t}{k} \leq 2 \sqrt{t\ln\frac{2tk}{\delta}} }^T  1  \right) + \left(1 + \frac{\pi}{2}\right)\delta T = \tilde{O}(k^\frac{3}{2} \sqrt{T \ln \frac{kT}{\delta}} + k^3)
   %     & \leq  k \left(\int_{t= 0}^{T}   \sqrt{\frac{\ln\frac{t}{\delta}}{\frac{t}{2k}}} + \int_{t=1}^{\tilde{O}(k^2\ln\frac{k}{\delta})}  1\right) + \delta T\\
  %      & \leq k^{\frac{3}{2}} \int_{t=1}^T \sqrt{\frac{\ln\frac{t}{\delta}}{t}} + \tilde{O}(k^3 \ln \frac{k}{\delta}) + \delta T\\
 %       & = k^{\frac{3}{2}}\sqrt{2T}\sqrt{\ln \frac{kT}{\delta}} + \tilde{O}(k^3 \ln \frac{k}{\delta}) + \delta T\\
%        & = \tilde{O}(k^\frac{3}{2} \sqrt{T \ln \frac{kT}{\delta}} + k^3) + \delta T\\
 %& = \tilde{O}(k^\frac{3}{2} \sqrt{T \ln \frac{kT}{\delta}} + k^3) \\
\end{align*}
 where this bound is derived in \app{sec:derivation}.
\end{proof}

%% file: noncontextual-lower.tex
\section{Fair Classic Stochastic Bandits: A Lower
  Bound}\label{sec:lower}

We now show that the regret bound for \fairbandits has an
optimal dependence on $k$: \emph{no} fair algorithm has diminishing
regret before $T = \Omega(k^3)$ rounds.  All missing proofs are in
\app{sec:app-lower}. The main result of this section is the following.

\begin{theorem}\label{thm:cubed-lower}
  There is a distribution $P$ over $k$-arm instances of the stochastic
  multi-armed bandit problem such that any fair algorithm run on $P$
  experiences constant per-round regret for at least 
\[ T=
  \Omega\left( k^3 \ln \frac{1}{\delta} \right)\]
 rounds.
\end{theorem}

Despite the fact that \emph{regret} is defined in a prior-free way,
the proof of Theorem~\ref{thm:cubed-lower} proceeds via Bayesian
reasoning. We construct a family of lower bound instances such that
arms have payoffs drawn from Bernoulli distributions, denoted $B(\mu)$
for mean $\mu$. So, to specify a problem instance, it suffices to
specify a mean for each of $k$ arms: $\mu_1,\ldots,\mu_k$. The proof
formalizes the following outline.
\begin{enumerate}
\item We define an instance distribution $P = P_1 \times \ldots \times P_k$ over
  means $\mu_i$ (Definition~\ref{def:prior}).  $P$ will have two
  important properties. First, we will draw means from $P$ such that for any $i\in [k-1]$, $\mu_i = \mu_{i+1}$ with probability
  at least $1/4$. Second, for any realization of means drawn from $P$,
  if an algorithm plays uniformly at random over $[k]$, it will suffer
  constant per-round regret.

\item We treat $P_i$ as a \emph{prior} distribution over mean $\mu_i$,
  and analyze the \emph{posterior} distribution
  $P_i(\rew{1}{i}, ,\ldots,\rew{t}{i})$ over means that results after
  applying Bayes' rule to the payoff observations
  $\rew{1}{i},\ldots,\rew{t}{i}$ made by the algorithm. Bayes' rule
  implies (Lemma~\ref{lem:bayes-exps}) the joint distribution over
  rewards and means drawn from $P$ is identical to the distribution
  which first draws means according to $P$, then draws rewards
  conditioned on those means, and finally \emph{resamples} the means
  from the posterior distribution on means. Thus, we can reason about
  fairness (a frequentist quantity) by analyzing the Bayesian
  posterior distribution on means conditioned on the observed rewards.

\item A $\delta$-fair algorithm, for any set of means realized
  from the instance (prior) distribution, must not play arm $i+1$ with
 lower probability than arm $i$ if $\mu_i = \mu_{i+1}$, except with
  probability $\delta$. By the above change of perspective, therefore,
  any $\delta$-fair algorithm must play arms $i$ and $i+1$ with equal
  probability until the \emph{posterior} distribution on means given
  observed rewards, satisfies $\pr{\mu_i = \mu_{i+1}|h} < \delta$
  (Lemmas~\ref{lem:distinguishing} and~\ref{lem:nonuniform}).

\item We finally lower bound the number of reward observations
  necessary before the posterior distribution on means given payoffs
  is such that $\pr{\mu_i = \mu_{i+1}|h} < \delta$ for any pair of
  adjacent arms $i, i+1$. We show that this is $\Omega(k^2)$
  (Lemma~\ref{lem:num-obs}). Since fair algorithms must play
  from among the $k$ arms uniformly at random until this point, with
  high probability, no arm accumulates sufficiently many reward
  observations until $T = \Omega(k^3)$ rounds of play.
\end{enumerate}

We begin by describing our distribution over instances. Each arm $i$'s
payoff distribution will be Bernoulli with mean $\mu_i \sim P_i$
independently of each other arm.
%For all realizations of the arms' means, it will be
%the case that $\mu_1 \geq \ldots \geq \mu_k$ -- but for any adjacent
%pair $i, i+1$, we will have $\pr{\mu_i = \mu_{i+1}} = \frac{1}{4}$.

\begin{definition}[Prior Distribution over $\mu_i$\label{def:prior}]
  For each arm $i$, $\mu_i$ is distributed according to the
  distribution with the following probability mass function:

\[
 P_i(x) =
  \begin{cases}
      \hfill \frac{1}{2}    \hfill & \text{ if $x =  \frac{1}{3} + \frac{i}{3k}$} \\
      \hfill \frac{1}{2} \hfill & \text{  if $x = \frac{1}{3} + \frac{i+1}{3k}$}. \\
  \end{cases}
\]

Let $P = \prod_i P_i$ denote the joint distribution on arms' expected payoffs.
\end{definition}

We treat $P$ as a prior distribution over instances, and analyze the
\emph{posterior} distribution on instances given the realized
rewards. Lemma~\ref{lem:bayes-exps} justifies this reasoning.
% it shows
%that the joint distribution on instances and observed rewards is
%identical independently of whether we draw the instance \emph{first}
%and then draw rewards, or if we draw the rewards first, and then draw
%the instance \emph{second} from the \emph{posterior} distribution on
%instances given the rewards.
%The proof is found in
%Appendix~\ref{sec:app-lower}.

\begin{lemma}\label{lem:bayes-exps}
  Consider the following two experiments: In the first, let
  $\mu_i \sim P_i$ and $\rew{1}{i}, \ldots, \rew{t}{i} \sim B(\mu_i)$,
  and $W$ denote the joint distribution on
  $(\mu_i, \rew{1}{i}, \ldots, \rew{t}{i})$.  In the second, let
  $\mu_i \sim P_i$, and
  $\rew{1}{i}, \ldots, \rew{t}{i} \sim B(\mu_i)$, and then re-draw the
  mean $\mu'_i \sim P_i (\rew{1}{i}, \ldots, \rew{t}{i})$ from its
  posterior distribution given the rewards. Let
  $ ( \mu'_i, \rew{1}{i}, \ldots, \rew{t}{i})\sim W'$.  Then, $W$ and
  $W'$ are identical distributions.
\end{lemma}

%Therefore, we can analyze the probability of a realization of means by
%analyzing their probability drawn from this posterior given
%observation of rewards.
Next, we lower-bound the number of reward observations necessary such
that for some $i\in [k]$: $\pr{\mu_i = \mu_{i+1}|h^t} < \delta$ with
respect to the posterior. It will be useful to refer to an algorithm's
histories as \emph{distinguishing} the mean of an arm given that
history with high probability.

\begin{definition}[$\delta$-distinguishing]\label{def:distinguish}
We will say $h^t$ $\delta$-{\bf distinguishes} arm $i$ for $\A$ if, for some
  $\alpha \in [0,1]$,

\[ \Pr{\mu'_i \sim P_i(h^t) }{\mu'_i = \alpha} \geq 1 - \delta. \]
\end{definition}

The next lemma
%(proof in Appendix~\ref{sec:app-lower})
 shows that if
no arm is $\sqrt{\delta}$-distinguished by a history, all pairs of
arms $i, i+1$ have posterior probability strictly greater than
$\delta$ of having equal means.

\begin{lemma}\label{lem:distinguishing}
  Suppose $\A$ has history $h^t$, and that  $h^t$ does not
  $\sqrt{\delta}$-distinguish any arm $i$. Then, for all arms $i, i+1$,
\[\pr{\mu_i = \mu_{i+1}|h^t} > \delta.\]
\end{lemma}
Now, we prove that for any fair algorithm, with probability
$\geq\tfrac{1}{2}$ over the draw of histories $h^t$, $h^t$ must
$\sqrt{2\delta}$-distinguish some arm, or the algorithm must play
uniformly across all $k$ arms conditioned on $h^t$.
%The proof can be
%found in Appendix~\ref{sec:app-lower}.
\begin{lemma}\label{lem:nonuniform}
  Suppose an algorithm $\A$ is $\delta$-fair. Then:
  % when playing
  %against the distribution over instances $P$,

\[\Pr{h^t\sim \A}{h^t\textrm{ does not $\sqrt{2\delta}$-distinguish
      any $i$} \wedge \exists t' \leq t, i\in [k] :
    \ppc{t'}{i'}{h^{t'}} \neq \frac{1}{k}} \leq \frac{1}{2}.\]
\end{lemma}

We now lower-bound the number of observations from arm $i$ which are
required to $\delta$-distinguish it.

\begin{lemma}\label{lem:num-obs}
  Fix any $\delta < \frac{1}{8}$. Let $\mu_i\sim P_i$ as in
  Definition~\ref{def:prior}. Then, arm $i$ is
  $\sqrt{2\delta}$-distinguishable by $h^t$ only if
  $T_i = \Omega(k^2\ln\frac{1}{\delta})$, where
  $T_i = | \{t' : h^{t'}_2 = i, t' \leq t\} |$ is the number of times
  arm $i$ is played.
\end{lemma}

\begin{proof}
  Write $p, p+ \frac{1}{3k}$ to represent the two possible
  realizations that $\mu_i$ might take, when drawn from the
  distribution over instances given in Definition~\ref{def:prior}. Let
  $A$ represent the event that $\mu_i = p$ and $B$ the event that
  $\mu_i = p + \frac{1}{3k}$. Let $\delta' = \sqrt{2\delta}$
  throughout.

  Fix a history $h^t$, and let $m =T _i$ represent the number of
  observations of arm $i$'s reward.  We will abuse notation and use
  $h^t_i$ to refer to the payoff sequence of arm $i$ observed in
  history $h^t$. $h^t_i$ is therefore a binary sequence of length $m$;
  let $||h^t_i||_0 = s$ denote the number of $1$s in the sequence. We
  will calculate conditions under which $h^t_i,m, s$ will imply that
  either
  $ \frac{1-\delta'}{\delta'} \leq \frac{\pr{B | h^t_i}}{\pr{A
      |h^t_i}}$
  or
  $ \frac{\pr{B | h^t_i}}{\pr{A |h^t_i}}
  \leq\frac{\delta'}{1-\delta'}$
  holds, implying that one of $A$ or $B$ has posterior probability at
  least $1-\delta'$, conditioned on the observed rewards. If neither
  of these is the case, $i$ is not $\delta'$-distinguished by $h^t$.

We begin by rearranging our definition of this ratio
%\begin{align*}
$X
%&
 =   \frac{\pr{B | h^t_i}}{\pr{A | h^t_i}} %\\
% & =   \frac{\frac{\pr{B \cap \data}}{\pr{\data}}}{\frac{\pr{A \cap \data}}{\pr{\data}}}\\
% & =   \frac{\pr{B \cap \data}}{\pr{A \cap \data}}\\
% & =   \frac{\pr{ \data | B}\pr{B}}{\pr{ \data | A }\pr{A}}\\
%&
 =   \frac{\pr{ h^t_i | B}}{\pr{ h^t_i | A }}$,%\]%\\
%\end{align*}
%where the second line follows from Bayes' rule, the second from
%cancellation, the third from Bayes' rule, and the final from
which follows from Bayes' rule and the fact that $\pr{A} = \pr{B}$. We
wish to upper and lower bound $X$ in terms of $h^t_i$'s value.  By
definition of the Bernoulli distribution, we have that
%
%\begin{align*}

\[ X   =  \frac{\pr{ h^t_i | B}}{\pr{h^t_i | A }}
  = \frac{\left(p + \frac{1}{3k}\right)^s \left(1- p - \frac{1}{3k}\right)^{m-s} }{\left(p\right)^s \left(1- p \right)^{m-s} }% \\
 = \left(1 + \frac{1}{3kp}\right)^s \left(1 - \frac{1}{3k(1-p)}\right)^{m-s}.\]
%\end{align*}
%
 We now calculate under what conditions either (a)
 $X \leq \frac{\delta'}{1-\delta'}$, or (b)
 $X \geq \frac{1-\delta'}{\delta'}$.  One of these must hold if $i$ is
 $\delta'$-distinguished.  Before we do so, we mention that a Chernoff
 bound implies that with probability $1-\delta'$, for events $A$ and
 $B$, Equations~\ref{eqn:chernoff-upper-one}
 and~\ref{eqn:chernoff-upper}, respectively:

\begin{minipage}{0.5\linewidth}
\begin{equation}
| s -mp | \leq \sqrt{2m\ln\frac{2}{\delta'}}
   \label{eqn:chernoff-upper-one}
\end{equation}
\end{minipage}%
\begin{minipage}{0.5\linewidth}
\begin{equation}
| s -mp -\frac{m}{3k}| \leq \sqrt{2m\ln\frac{2}{\delta'}} \label{eqn:chernoff-upper}
\end{equation}
\end{minipage}

\noindent since the mean of $m$ Bernoulli trials with mean $p$ (
or $p + \frac{1}{3k}$) is $mp$ (or $mp + \frac{m}{3k}$).

We begin by analyzing case (a), where $\delta' = \sqrt{2\delta} < 1/2$ implies

\[2\delta' > \frac{\delta'}{1-\delta'} \geq
X   = \left(1 + \frac{1}{3kp}\right)^s \left(1 - \frac{1}{3k(1-p)}\right)^{m-s}.
\]
Taking logarithms on both sides, we have that

\[
\ln(2\delta') >    s\ln \left(1 + \frac{1}{3kp}\right) + (m-s)\ln \left(1 - \frac{1}{3k(1-p)}\right)% \geq s \frac{\frac{1}{3kp}}{1 + \frac{1}{3kp}} +  (m-s)\frac{\frac{-1}{3k(1-p)}}{1 - \frac{1}{3k(1-p)}} \\
\geq s\frac{1}{3kp + 1} - (m-s)\frac{1}{3k(1-p) - 1}
\]
where the inequality follows from $\ln(1 + x) \geq \frac{x}{x+1}$ for
$x \in [-1, \infty]$. Then, this implies that
\begin{align*}
(3kp+1)(3k(1-p)-1) \ln(2\delta')
 & >    s(3k(1-p) - 1) - (m-s)(3kp+1) %\\
% &  =  3ks(1-p) - s - (m-s)3kp - (m-s)\\
%& = 3ks - 3ksp -s -3kpm +3kps - m + s\\
& =  3ks - 3kpm - m.
\end{align*}
Multiplying both sides by $-1$, this implies that
\[ (3kp+1)(3k(1-p)-1)  \ln\frac{1}{2\delta'} < m + 3kpm - 3ks.\]
Equation~\ref{eqn:chernoff-upper-one} implies
$|3ks - 3kmp - m| \leq m + 3k\sqrt{2m\ln\frac{2}{\delta'}}$, which
with the previous line implies
\[(3kp+1)(3k(1-p)-1) \ln\frac{1}{2\delta'} < m +
3k\sqrt{2m\ln\frac{2}{\delta'}}.\]
Since $p, 1-p \in [1/3,2/3]$ and $\delta' = \sqrt{2\delta}$, solving
for $m$ implies that $m = \Omega(k^2 \ln\frac{1}{\delta'})$.

In case (b), we have
\[
\frac{1}{2\delta'} < \frac{1-\delta'}{\delta'} \leq
X   = \left(1 + \frac{1}{3kp}\right)^s \left(1 - \frac{1}{3k(1-p)}\right)^{m-s}  \leq e^{\frac{s}{3kp}}e^{\frac{-(m-s)}{3k(1-p)}}
\]
where we used the fact that $1 + x \leq e^x$ for all $x$.
Taking logarithms, this will imply that
\begin{align}
\frac{s}{3kp} - \frac{m-s}{3k(1-p)} > \ln\frac{1}{2\delta'}  \Rightarrow
s -mp =  s(1-p)  - (m-s)p >  3kp(1-p) \ln\frac{1}{2\delta'} \geq \frac{6k}{9} \ln\frac{1}{2\delta'} \label{eqn:two-lower}\end{align}
whose last inequality comes the range of $p$.  Combining this
inequality with Equation~\ref{eqn:chernoff-upper}, this implies
\[\sqrt{2m\ln\frac{1}{2\delta'}} + \frac{m}{k} > \frac{6k}{9}\ln\frac{1}{2\delta'}\]
and solving for $m$ and substituting for $\delta'$ gives that
$m = \Omega(k^2 \ln \frac{1}{\delta})$.

Thus, if either $X \geq \frac{1-\delta'}{\delta'}$ or
$X \leq \frac{\delta'}{1-\delta'}$, it must be that
$m = \Omega(k^2 \ln \frac{1}{\delta})$.
\end{proof}

We now have the tools in hand to prove
Theorem~\ref{thm:cubed-lower}. \iffalse We will first show that with
probability at least $1/2$ over histories, $\A$ must play uniformly
over $k$ arms for $\Omega(k^3)$ rounds, and that in those
$\Omega(k^3)$ rounds, $\A$'s per-round regret is constant.\fi
\begin{proof}[Proof of Theorem~\ref{thm:cubed-lower}]
  Assume $\A$ is some $\delta$-fair algorithm where $\delta < 1/8$.
  Fix $T$; we claim that with probability at least $\frac{1}{2}$, for
  any $t = o(k^3\ln\frac{1}{\delta})$, $t\leq T$, 
  $\ppc{t}{j}{h^t} = \frac{1}{k}$ for all $j$. Since the payoff for
  uniformly random play is $\leq \frac{1}{2} + \frac{1}{k}$, while the
  best arm has payoff $\geq \frac{2}{3}$, in any round $t$ where
  $\ppc{t}{i}{h^{t}} = \ppc{t}{i'}{h^{t}}$ for all $i, i' \in [k]$,
  the algorithm suffers $\Omega(1)$ regret in that round.

  Lemma~\ref{lem:nonuniform} implies that, with probability at least
  $\frac{1}{2}$ over the distribution over histories $h^t$, either (a)
  $\ppc{t'}{i}{h^{t'}} = \ppc{t'}{i'}{h^{t'}}$ for all
  $i, i' \in [k], t' \leq t$ or (b) $h^t$ must
  $\sqrt{2\delta}$-distinguish some arm $i$. Case $(a)$ implies our
  claim. In case (b), Lemma~\ref{lem:num-obs} states than an arm $i$
  is $\sqrt{2\delta}$-distinguishable only if
  $T_i = \Omega(k^2\ln\frac{1}{\delta})$. We now argue that unless
  $t = \Omega(k^3 \ln\frac{1}{\delta})$,
  $T_i = o(k^2 \ln\frac{1}{\delta})$, which will imply our claim for
  case $(b)$.

  Fix some $i, t$. We lower-bound $t$ for which, with probability at
  least $1-\frac{\delta'}{k}$ over histories $h^t$, it will be the
  case that $\num{t}{i} \geq c \cdot k^2\ln\frac{1}{\delta}$ when
  $\ppc{t'}{i}{h^{t'}} = \ppc{t'}{i'}{h^{t'}}$ for all
  $i, i' \in [k], t' \leq t$.  Let $X_1, \ldots, X_t$ be indicator
  variables of arm $i$ being played in round $t' \leq t$.  Note that
  for all $t'\leq t$, $E[X_{t'}] = \frac{1}{k}$, since in all rounds
  prior to $t$, we have all arms are played with equal
  probability. For any $\epsilon \in [0,1]$, as $\num{t'}{i}$ are
  nondecreasing in $t'$, an additive Chernoff bound implies
%
% \max_{t' \leq t} \pr{\sum_{t'' \leq t'}X_{t''} > \frac{t'}{k} - \epsilon t} \leq
\[ \pr{\exists t'\leq t : \num{t'}{i} \geq \frac{t}{k} + \epsilon t} \leq  \pr{\sum_{t' \leq t}X_{t'} > \frac{t}{k} + \epsilon t}\leq e^{-2 t\epsilon^2}\]
which, for $\epsilon t = \sqrt{\frac{t \ln \frac{2k}{\delta'}}{2}}$,
becomes
$ \pr{\sum_{t' \leq t}X_{t'} > \frac{t}{k} + \epsilon t}\leq
\frac{\delta'}{k}.$
So, using a union bound over all $k$ arms, with probability
$1-\delta'$, for some fixed $t$ and all $i$,
$\num{t}{i} \leq \frac{t}{k} + \sqrt{\frac{t \ln
    \frac{2k}{\delta'}}{2}}$.
We condition on the event that $\num{t}{i}$ satisfies this inequality
for a fixed $t$ and all $i$.  If
$\num{t}{i} \geq c \cdot k^2 \ln \frac{1}{\delta}$, this implies
\[ \frac{t}{k} + \sqrt{\frac{t \ln  \frac{2k}{\delta'}}{2}} \geq c \cdot k^2 \ln \frac{1}{\delta}
\quad \Rightarrow \quad t \geq - k \sqrt{\frac{t \ln  \frac{2k}{\delta'}}{2}} +  c \cdot k^3 \ln \frac{1}{\delta}.\]
If
$k \sqrt{\frac{t \ln \frac{2k}{\delta'}}{2}} \leq \frac{c}{2} \cdot
k^3 \ln \frac{1}{\delta}$,
then  $t \geq \frac{c}{2} \cdot k^3 \ln \frac{1}{\delta}$;
if not, then
$t \geq \frac{\frac{c^2}{2} k^4
  \ln^2\frac{1}{\delta}}{\ln\frac{2k}{\delta'}}$.
Thus, $\num{t}{i} < c \cdot k^2 \ln \frac{1}{\delta}$ with probability
$1-\delta'$ for all $i$ unless
$t \geq \min \left( \frac{c}{2} \cdot k^3 \ln \frac{1}{\delta},
  \frac{\frac{c^2}{2} k^4
    \ln^2\frac{1}{\delta}}{\ln\frac{2k}{\delta'}}\right)=
\Omega(k^3\ln\frac{1}{\delta})$ for $\delta' \in [\frac{1}{2}, 1]$.
\end{proof} 

%% file: kwiktofair.tex
\newcommand{\s}[2]{s^{#1}_{#2}}
\section{KWIK Learnability Implies Fair Bandit Learnability}\label{sec:kwiktofair}

In this section, we show if a class of functions is KWIK learnable,
then there is a fair algorithm for learning the same class of
functions in the contextual bandit setting, with a regret bound
polynomially related to the function class' KWIK bound.  Intuitively,
KWIK-learnability of a class of functions guarantees we can learn the
function's behavior to a high degree of accuracy with a high degree of
confidence.  As fairness constrains an algorithm most before the
algorithm has determined the payoff functions' behavior accurately,
this guarantee enables us to learn fairly without incurring much
additional regret. Formally, we prove the following polynomial
relationship.

\begin{theorem} \label{thm:kwiktofair} For an instance of the contextual
  multi-armed bandit problem  where $f_j\in C$ for all $j\in [k]$, if $C$ is
  $(\epsilon, \delta)$-KWIK learnable with bound
  $m(\epsilon, \delta)$, \kwikfair$(\delta, T)$ is $\delta$-fair and
  achieves regret bound:
  $$R(T) = O\left(\max\left(k^2 \cdot m\left(\epsilon^*,\frac{\min\left(\delta, 1/T\right)}{T^2
    k}\right), k^3 \ln\frac{k}{\delta}\right)\right)$$ for $\delta \leq \frac{1}{\sqrt{T}}$
  where
  $\epsilon^* = \arg\min_\epsilon(\max(\epsilon\cdot T, k\cdot
  m(\epsilon,\frac{\min(\delta, 1/T)}{kT^2}))).$
\end{theorem}

%We now outline our argument.
First, we construct an algorithm
\textsc{\kwikfair}$(\delta,T)$ that uses the KWIK learning
algorithm as a subroutine, and prove that it is $\delta$-fair. A call
to \textsc{\kwikfair}$(\delta,T)$ will initialize a KWIK
learner for each arm, and in each of the $T$ rounds will implicitly
construct a confidence interval around the prediction of each learner.
If a learner makes a numeric valued prediction, we will interpret this
as a confidence interval centered at the prediction with width
$\epsilon^*$. If a learner outputs $\bot$, we interpret this as a
trivial confidence interval (covering all of $[0,1]$).  We use the same chaining technique that we use in the classic stochastic setting. In every round
$t$, \kwikfair$(\delta, T)$ identifies the arm
$\ist = \argmax_i\up{t}{i}$ that has the largest \emph{upper}
confidence bound. At each round $t$, we will say $i$ is \emph{linked}
to $j$ if
$[\low{t}{i}, \up{t}{i} ]\cap [\low{t}{j}, \up{t}{j}] \neq \emptyset$,
and $i$ is \emph{chained} to $j$ if they are in the same component of
the transitive closure of the linked relation. Then, it plays
uniformly at random amongst all arms chained to arm $\ist$. Whenever
all learners output predictions, they need no feedback. When a learner
for $j$ outputs $\bot$, \emph{if} $j$ is selected then we have
feedback $\rew{t}{j}$ to give it; on the other hand, if $j$ isn't
selected, we ``roll back'' the learning algorithm for $j$ to before
this round by not updating the algorithm's state.

\begin{center}
	\begin{algorithmic}[1]
	\Procedure{\textsc{\kwikfair}}{$\delta, T$}
 	\State $\delta^* \gets \frac{\min(\delta, \frac{1}{T})}{k T^2}$, $\epsilon^* \gets \arg\min_\epsilon(\max(\epsilon\cdot T, k\cdot
  m(\epsilon,\delta^*)))$
 %		\For{$i = 1, \ldots, k$}
                \State Initialize KWIK$(\epsilon^*,\delta^*)$-learner $L_i$,  $h_i \gets [\,]$ $\forall i \in [k]$% \Comment{The state of queries and answers for $L_i$}
 %		\EndFor
 		\For{$1 \leq t \leq T$}
 		\State S $\gets \emptyset$ \Comment{Initialize set of predictions $S$}
 			\For{$i = 1, \ldots, k$}
 				\State $\s{t}{i} \gets L_i(x_i^t, h_i)$
 				\State $S \gets S \cup \s{t}{i}$ \Comment{Store prediction $\s{t}{i}$}
 			\EndFor
 			\If{$\perp \in S$}
 				\State Pull $j^* \gets (x \in_R [k])$, receive reward  $\rew{t}{j^*}$ \Comment{Pick arm at random from all arms}
% 				\State Pull arm $j^*$ \Comment{Pull arm at random}
 			\Else
				\State $\ist \gets \arg \max_i \s{t}{i}$
				\State $S^{t} \gets \{j \mid (\s{t}{j} - \epsilon^*, \s{t}{j} + \epsilon^*) \text{ chains to } (\s{t}{\ist} - \epsilon^*, \s{t}{\ist} + \epsilon^*)\}$
				\State Pull $j^* \gets (x \in_R S^t)$, receive reward  $\rew{t}{j^*}$ \Comment{Pick arm at random from active set $\s{t}{i^*}$}
%				\State Pull arm $j^*$ \Comment{Pull arm at random from active set $\s{t}{i^*}$}
 			\EndIf
    %                    \State  Pull arm $j^*$, receive reward $\rew{t}{j^*}$ \Comment{Pull arm}
                        \State $h_{j^*} \gets h_{j^*} :: (x_{{j^*}}^t, \rew{t}{j^*})$ \Comment{Update the
 history for $L_{j^*}$}
 		\EndFor
 	\EndProcedure
 	\end{algorithmic}
\end{center}

We begin by bounding the probability of certain failures of
\textsc{\kwikfair} in Lemma \ref{lem:prfailure}, proven in
\app{sec:app-missing}. This in turn lets us prove the fairness of
\kwikfair in Theorem \ref{thm:kwiktofairfixed}.

\begin{lemma}\label{lem:prfailure}
  With probability at least $1-\min(\delta, \frac{1}{T})$, for
  all rounds $t$ and all arms $i$, (a) if $\s{t}{i} \in \R$ then
  $|\s{t}{i} - f_i(x^t_i)| \leq \epsilon^*$ and (b)
  $\sum_{t} \I{\s{t}{i} = \bot \text{ and $i$ is pulled}}
  \leq m(\epsilon^*, \delta^*)$.
\end{lemma}

\begin{theorem} \label{thm:kwiktofairfixed} \textsc{\kwikfair}$(\delta,T)$ is
  $\delta$-fair.
\end{theorem}
\begin{proof}[Proof of Theorem~\ref{thm:kwiktofairfixed}]
  We condition on both (a) and (b) holding for all arms $i$ and rounds
  $t$ from Lemma~\ref{lem:prfailure}, which occur with probability
  $1-\delta$ for all arms and all times $t$.  Therefore, we proceed by
  conditioning on the event that for all arms $i$ and all rounds $t$,
  if $L_i = \s{t}{i}$ for $\s{t}{i} \neq \bt$ then
  $|\s{t}{i} - f_i(x^t_i)| \leq \epsilon^*$.  Having done so, there are
  two possibilities for each round $t$.

  In case 1, for each $i$ we have that
  $L_i(x^t_i)= \s{t}{i}\neq \bot$. By the condition above, for any
  arms $i$ and $j$, $f_i(x_i^t) \geq f_j(x_j^t)$ implies that
  $\s{t}{i} + \epsilon^* \geq \s{t}{j} - \epsilon^*$. Since in this case
  no learner outputs \bt, arm $j$ chains to the top arm only if arm
  $i$ does. Therefore $\pi_{i \mid h}^t \geq \pi_{j \mid h}^t$. In case 2,  there exists some $i$ such that
  $L_i(x^t_i) = \bt$. Then we choose uniformly at random across all
  arms, so $\pi_{i \mid h}^t = \pi_{j \mid h}^t$ for all $i$ and
  $j$.

Thus, with probability at least $1 -\delta$, for each round $t$,
$f_i(x_i^t) \geq f_j(x_j^t)$ implies that
$\pi_{i \mid h}^t \geq \pi_{j \mid h}^t$. \iffalse Thus
\textsc{\kwikfair}$(\delta,T)$ is $\delta$-fair.\fi
\end{proof}

We now use the KWIK bounds of the KWIK learners to upper-bound the
regret of \textsc{\kwikfair}$(\delta,T)$.
\begin{lemma} \label{lem:kwiktofairfixedregret} \textsc{\kwikfair}$(\delta,T)$
  achieves regret
  $O(\max(k^2 \cdot m(\epsilon^*,\delta^*), k^3\ln\frac{Tk}{\delta}))$.
\end{lemma}

\begin{proof}
  We first condition on the event that both (a) and (b) from
  Lemma~\ref{lem:prfailure} hold for all $t,i$, which holds with
  probability $1-\min(\delta, \frac{1}{T})$, and bound the
  regret when they both hold.  Choose an arbitrary round $t$ in the
  execution of \textsc{\kwikfair}$(\delta,T)$. As above, there are two
  cases. In the first case, $L_i(x^t_i) = \s{t}{i} \neq \bt$ for all
  $i$ and we choose uniformly at random from the arms chained by
  $\epsilon^*$-intervals to the arm with the highest prediction. Since
  we have conditioned on the event that all KWIK learners are correct,
  $i^* \in S^t$.  Furthermore, for any $i, j\in S^t$, we have that
  $|\s{t}{i} - \s{t}{j}| \leq 2k\epsilon^*$, and in particular that
  $|\s{t}{i} - \s{t}{i^*}| \leq 2k\epsilon^*$. Thus, the regret is at
  most $2k\epsilon^*$ in such a round.  In the second case some arm
  outputs \bt, so we choose randomly from all $k$ arms, and the
  worst-case regret is 1. Thus, the total regret will be at most
  $2k\epsilon^* T + n + \delta T$ where $n$ is the number of rounds in
  which some $L_i$ outputs \bt.

  We now upper bound $n_i$, the number of rounds in which arm $i$
  outputs \bt. Fix some arm $i$ which outputs \bt in $n_i$ rounds.
  Arm $i$ is played and therefore receives feedback every time it
  outputs $\bot$ with probability at least $1/k$. Thus, using a
  Chernoff bound, with probability $1-\delta'$, arm $i$ receives
  feedback for $n_i$ outputs of $\bot$ in at least
  $\frac{n_i}{k} - \sqrt{2n_i \ln \frac{2}{\delta'}}$ rounds. $L_i$
  has the guarantee that there can be at most $m(\epsilon^*, \delta^*)$
  many such rounds (in which it outputs $\bot$ \emph{and} receives feedback).
  Thus,
\[ m(\epsilon^*, \delta^*) \geq \frac{n_i}{k} -
  \sqrt{2n_i \ln \frac{2}{\delta'}}.\]
  If $n_i \geq ck \cdot m(\epsilon^*, \delta^*)$, this implies
  \[ m(\epsilon^*, \delta^*) \geq c \cdot m(\epsilon^*,
  \delta^*) - \sqrt{2 ck \cdot m(\epsilon^*, \delta^*) \ln
    \frac{2}{\delta'}}.\] We now analyze cases in which (1)
  $k \ln \frac{2}{\delta'} \leq m(\epsilon^*, \delta^*)$ and (2)
  $k \ln \frac{2}{\delta'} > m(\epsilon^*, \delta^*)$.

Case (1) this
  implies
\[ m(\epsilon^*, \delta^*) \geq c \cdot m(\epsilon^*, \delta^*) -
 \sqrt{2 c} \cdot m(\epsilon^*, \delta^*).\]
 For $c > 4$, this leads to contradiction. Thus, in this case, if we
 set $\delta' = \frac{\delta}{k}$, we know that with probability
 $1-\delta$, $n_i \leq 4k \cdot m(\epsilon^*, \delta^*)$ which summing
 up over all $i$ implies
 $\sum_i n_i \leq 4k^2 \cdot m(\epsilon^*, \delta^*)$, as desired.

In case (2), we have that
\[ \ln \frac{2}{\delta'} \geq \frac{n_i}{k} - \sqrt{2n_i \ln
  \frac{2}{\delta'}}\]
which solving for $n_i$ implies that
$n_i = O(k^2 \ln\frac{1}{\delta'})$, so
$\sum_i n_i = O(k^3 \ln\frac{k}{\delta})$ by setting
$\delta' = \frac{\delta}{k}$ and taking a union bound.  Thus, there
are at most
$n = O(\max (k^2 \cdot m(\epsilon,\delta^*), k^3 \ln\frac{k}{\delta}))$
rounds in expectation during the execution of \textsc{\kwikfair}$(\delta,T)$
in which some arm outputs \bt.

Combining both cases, the total regret incurred by
\textsc{\kwikfair}$(\delta,T)$ across all $T$ rounds is
\begin{align*}
  R(T) = 4 k^2 \cdot m(\epsilon^*, \delta^*) +  k^3\ln\left(\frac{k}{\delta}\right) + T\cdot 2k\epsilon^* + T \cdot \min(\delta, \frac{1}{T})%  \leq 4k^2 \cdot (T \cdot \epsilon^*/k) +  k^3\ln\frac{k}{\delta})  +  T \cdot 2k\epsilon^* +T \min(\delta, \frac{1}{T}) \\
%       & = 3 \cdot Tk\epsilon^* +  k^3\ln\frac{k}{\delta} + 1
	 = O(\max(k^2 \cdot m(\epsilon^*,\delta^*),k^3\ln\frac{k}{\delta})).
\end{align*}
\end{proof}

Our presentation of \textsc{\kwikfair}$(\delta,T)$ has a
known time horizon $T$. Its guarantees extend to the case in which $T$
is unknown via the standard ``doubling trick'' to prove Theorem
\ref{thm:kwiktofair} in \app{sec:app-missing}.

An important instance of the contextual bandit problem is the linear
case, where $C$ consists of the set of all \emph{linear} functions of
bounded norm in $d$ dimensions.  This captures the natural setting in
which the rewards of each arm are governed by an underlying linear
regression model on a $d$-dimensional real valued feature space. The
linear case is well studied, and there are known KWIK
algorithms~\citep{strehl2008online} for the set of linear functions
$C$, which allows us via our reduction to give a fair contextual
bandit algorithm for this setting with a polynomial regret bound.

\begin{lemma}[\citep{strehl2008online}] \label{lem:kwiktofairlinear}
  Let
  $C = \{f_\theta| f_\theta(x) = \langle \theta, x\rangle, \theta\in
  \R^d, ||\theta|| \leq 1\}$
  and $\cX = \{x \in \mathbb{R}^d : ||x|| \leq 1\}$.  $C$
  is KWIK learnable with KWIK bound
  $m(\epsilon,\delta) = \tilde O(d^3/\epsilon^4)$.
\end{lemma}

Then, an application of Theorem~\ref{thm:kwiktofair} implies that
\kwikfair has a polynomial regret guarantee for the class of
linear functions.  This proof can be found in \app{sec:app-missing}.

\begin{corollary} \label{cor:kwiktofairlinear} Let $C$ and $\cX$ be as
  in Lemma \ref{lem:kwiktofairlinear}, and $f_j \in C$ for each
  $j\in [k]$.  Then, \textsc{\kwikfair}$(T, \delta)$ using the learner from
  \citep{strehl2008online} has regret:
  $$R(T) = \tilde O\left(\max \left(T^{4/5}k^{6/5}d^{3/5} ,
  k^3\ln\frac{k}{\delta}\right)\right).$$
\end{corollary} 

%% file: fairtokwik.tex
\newcommand{\ndisc}[1]{\ensuremath{\lceil \frac{1}{#1} \rceil}}
\newcommand{\sdisc}[1]{\ensuremath{[\ndisc{#1}]}}
\newcommand{\edisc}{\ensuremath{\ndisc{\epsilon}}}
\newcommand{\esdisc}{\ensuremath{\sdisc{\epsilon}}}
\newcommand{\fs}{f^{*}}
\newcommand{\lep}{\lceil \frac{1}{\epsilon^*}\rceil}
\renewcommand{\l}{\ell}
\newcommand{\sep}{\left[ \lep\right]}

\section{Fair Bandit Learnability Implies KWIK Learnability} \label{sec:fairtokwik}

In this section, we show how to use a fair, no-regret contextual
bandit algorithm to construct a KWIK learning algorithm whose KWIK
bound has logarithmic dependence on the number of rounds
$T$. Intuitively, any fair algorithm which achieves low regret must
both be able to find and exploit an optimal arm (since the algorithm
is no-regret) \emph{and} can only exploit that arm once it has a tight
understanding of the qualities of all arms (since the algorithm is
fair). Thus, any fair no-regret algorithm will \iffalse first explore
uniformly at random amongst the arms but will \fi ultimately have
tight $(1-\delta)$-confidence about each arm's reward function.

\begin{theorem}\label{thm:fairtokwik}
  Suppose $\A$ is a $\delta$-fair algorithm for the contextual bandit
  problem over the class of functions $C$, with regret bound
  $R(T, \delta)$. Suppose also there exists $f \in C, x(\ell)\in \cX$
  such that for every $\ell \in [\lceil\frac{1}{\epsilon}\rceil]$,
  $f(x(\ell)) = \ell \cdot \epsilon$.  Then, \fairkwik is an
  $(\epsilon, \delta)$-KWIK algorithm for $C$ with KWIK bound
  $m(\epsilon,\delta)$, with $m(\epsilon,\delta)$ the solution to
  $\frac{m(\epsilon, \delta) \epsilon}{4} = R(m(\epsilon, \delta),
  \frac{\epsilon\delta}{2T})$.
\end{theorem}

\begin{remark}
  The condition that $C$ should contain a function that can take on
  values that are multiples of $\epsilon$ is for technical
  convenience; $C$ can always be augmented by adding a single such
  function.% to satisfy the condition.
\end{remark}

Our aim is to construct a KWIK algorithm $\B$ to predict labels for a
sequence of examples labeled with some unknown function $f^* \in C$.
To do this, we will run our fair contextual bandit algorithm $\A$ on
an instance that we construct online as examples $x^t$ arrive for
$\B$. The idea is to simulate a two arm instance, in which one arm's
rewards are governed by $f^*$ (the function to be KWIK learned), and
the other arm's rewards are governed by a function $f$ that we can set
to take any value in $\{0,\epsilon,2\epsilon,\ldots,1\}$. For each
input $x^t$, we perform a thought experiment and consider $\A$'s
probability distribution over arms when facing a context which forces
arm $2$'s payoff to take each of the values
$0, \epsilon^*, 2\epsilon^*, \ldots, 1$.  Since $\A$ is fair, $\A$
will play arm $1$ with weakly higher probability than arm $2$ for
those $\l : \l\epsilon^* \leq f(x^t)$; analogously, $\A$ will play arm
$1$ with weakly lower probability than arm $2$ for those
$\l : \l\epsilon^* \geq f(x^t)$.  If there are at least $2$ values of
$\l$ for which arm $1$ and arm $2$ are played with equal probability,
one of those contexts will force $\A$ to suffer $\epsilon^*$ regret,
so we continue the simulation of $\A$ on one of those instances
selected at random, forcing at least $\epsilon^*/2$
regret in expectation, and at the same time have $\B$ return
$\bot$. $\B$ receives $\fs(x^t)$ on such a round, which is used to
construct feedback for $\A$.  Otherwise, $\A$ must transition from
playing arm $1$ with strictly higher probability to playing $2$ with
strictly higher probability as $\l$ increases: the point at which that
occurs will ``sandwich'' the value of $f(x^t)$, since $\A$'s fairness
implies this transition must occur when the expected
payoff of arm $2$ exceeds that of arm $1$. $\B$ uses this value to
output a numeric prediction.

An important fact we exploit is that we can \emph{query} $\A$'s
behavior on $(x^t, x(\l))$, for any $x^t$ and $\l\in \sep$ without
providing it feedback (and instead ``roll back'' its history to $h^t$
not including the query $(x^t, x(\l))$).  We update $\A$'s history by
providing it feedback only in rounds where $\B$ outputs $\bot$.

%As $\A$ has bounded regret and suffers expected regret
%$\Omega(\epsilon^*)$ in such rounds, there cannot be too many such
%rounds.
%

\begin{center}
	\begin{algorithmic}[1]
	\Procedure{\fairkwik}{$\epsilon, \delta, T$}
 	\State $\epsilon^* \gets \frac{\epsilon}{2}$,  $\delta^* \gets \frac{\delta \cdot \epsilon^*}{T}$, $h \gets [\,]$, initialize fair $\A(\delta^*, T)$ for class $C$
 		\For{$1 \leq t \leq T$}
                \State $h^t \gets h$
 		\For{$\ell\in \sep$}
                \State Let $(p^{t,\ell}_1, p^{t,\ell}_2) =( \ppc{}{1}{h^t},  \ppc{}{2}{h^t})) $ be $\A(h^t, (x^t, x(\l)))$'s  dist. over arms given $h$
 		\EndFor
                \If{$\exists \l \neq \l' : \l, \l' \in \sep$: $p^{t, {\l}}_1 =  p^{t, {\l}}_2 = p^{t, {\l'}}_1 =  p^{t, {\l'}}_2$} \\
% \Comment{There are 2 nonadjacent $(x^t, x(n^1 \epsilon)), (x^t, x(n^2 \epsilon)) $ contexts in which $p_1 = p_2$}
                \State Choose $x(\hat{\l}) \in_R \{x(\l),x(\l') \}$ \Comment{One of $x(\l), x(\l')$ must cause  $\epsilon^*$ regret}
                \State Select $a^t \sim \A(h^t, (x^t,x(\hat{\l})) )$ \Comment{Run $\A$ to get a predicted arm}
                \State  Predict $\hat{y}^t \gets \bot$
%                \If{ $a^t = 1$}
                \If{$a^t = 1$}
                \State $ \rew{t}{1} \gets y^t$ and  $h \gets h^t :: ((x^t, x(\hat{\l})), 1, \rew{t}{1})$ \Comment{Use KWIK feedback}% for $f^*$ as feedback for arm $1$}
                \Else
                \State $\rew{t}{2} \gets \hat{\l}\epsilon^*$
                \State $h \gets h^t :: ((x^t, x(\hat{\l})), 2, \rew{t}{2})$ \Comment{Construct feedback for arm $2$}
                \EndIf
%                \State  $h \gets h^t :: ((x^t, x), a^t, \rew{t}{a^t} )$  \Comment{Update $\A$'s history}
                \Else % \Comment{There is at most one $n^a$ such that $p^{t, {n^a}}_1 =  p^{t, {n^a}}_2 = p^{t, {n^{a+1}}}_1 =  p^{t, {n^{a+1}}}_2$}
                \Comment{$\A$'s history is not updated}
                \If{$p^{t, \l}_1 \leq p^{t, \l}_2$ for all $\l \in \sep$}
                \State $\hat{\l} \gets 0$
                \Else\textrm{ } Let $\hat{\l}$ be the largest index for which $ p^{t, \hat{\l}}_1 > p^{t, \hat{\l}}_2$
                \EndIf
                \State Predict $\hat{y}^t \gets \hat{\l} \cdot \epsilon^*$
                \EndIf
 		\EndFor
 	\EndProcedure
 	\end{algorithmic}
\end{center}

%We treat arm $1$ as the unknown $\fs\in C$ which $\B$ needs to learn,
%and arm $2$ as $f\in C$ which can take on any value in an
%$\epsilon^*$-net.
%When $\A$ makes a numeric prediction of $\fs(x^t)$,
%we need to show that prediction is accurate up to additive
%$\epsilon^*$. $\A$'s fairness implies this, by ``sandwiching''
%$\fs(x^t)$ whenever $\hat{\l}$ is calculated. Furthermore, when $\B$
%outputs $\bot$, this corresponds to two contexts
%$(x^t, x(\l)), (x^t, x(\l'))$ for which $\A$ plays the two arms each
%with probability $\frac{1}{2}$. Since arm $2$'s value for
%$f(x(\l)) = \l\epsilon^*$ and $f(x(\l')) = \l'\epsilon^*$ are at least
%$\epsilon^*$ apart, and $f(x^t)$ is fixed, one of these contexts
%causes at least $\epsilon^*$ regret in that round. Thus, randomizing
%over $(x^t, x(\l)), (x^t, x(\l'))$ will force $\A$ to have expected
%regret at least $\frac{\epsilon^*}{2}$. Since $\A$'s regret is
%bounded, this will bound the number of times the latter case can
%occur.
\vspace{-.15in}
\begin{proof}
  For a fixed run of $\A$, we calculate the probability that for
  all times $t$ and $\l\in \sep$, it is the case that
  $p^{t, \l}_1 > p^{t, \l}_2$ only if $\fs(x^t) > \l \cdot \epsilon^*$
  and also $p^{t, \l}_1 < p^{t, \l}_2$ only if
  $\fs(x^t) < \l \cdot \epsilon^*$.  In this run, $\A$ is queried on
  $\frac{T}{\epsilon^*}$ histories and contexts: prefixes of $h$ along
  with $(x^t, x(\l))$ for each $t\in [T], \l\in \sep$.  The fairness
  of $\A$ implies for any fixed $h^t$ and fixed $(x^t, x(\l))$, with
  probability $1-\delta^*$, $p^{t, \l}_1 > p^{t, \l}_2$ only if
  $\fs(x^t) > \l \epsilon^*$ and $p^{t, \l}_1 < p^{t, \l}_2$
  only if $\fs(x^t) < \l \epsilon^*$. Then, by a union bound
  over $t\in [T] ,\l\in \sep$, with probability at least
  $1 - \delta^* \frac{T}{\epsilon^*}= 1-\delta$, $\A(h^t, x^t, x(\l))$
  will satisfy this property for all $t\in [T], \l\in \sep$. We
  condition on this holding in the remainder of the proof.

  We now argue that the numeric predictions of $\B$ are correct within
  an additive $\epsilon$. Let:
  $$E^t = \{ \l : p^{t, \l}_1 = p^{t, \l}_2\}.$$  When
  $\B(x^t) = y^t \in [0,1]$, note that $|E^t| \leq 1$, else $\B$ would
  have output $\bot$.

  If $p^{t, \l}_1 \leq p^{t, \l}_2$ for all $\l$, since
  $|E^t| \leq 1$, either $p^{t,0}_1 < p^{t,0}_2$ or
  $p^{t,1}_1 < p^{t,1}_2$, which we have conditioned on implying that
  either $\fs(x^t) < f(x(0)) = 0$ or
  $\fs(x^t) < f(x(1)) = \epsilon^*$. Since $\fs(x^t) \geq 0$, this
  implies
  $\fs(x^t) \in [0, \epsilon^*) = [\hat{\l}\epsilon^*, \epsilon^*) =
  [\hat{y}^t, \hat{y}^t + \epsilon^*)$.\iffalse
  so in this case $\hat{y}^t$ is $2\epsilon^*$ accurate.\fi

  Otherwise, we have that $p^{t,\hat{\l}}_1 > p^{t,\hat{\l}}_2$, and
  $p^{t,\l}_1 \leq p^{t,\l}_2$ for all $\l > \hat{\l}$.  If $(a)$
  $\hat{\l} = \lep$, then $\fs(x^t) > 1$, a contradiction, so
  $\hat{\l} < \lep$. If $(b)$ $\hat{\l} = \lep -1$, then
  $\fs(x^t) > f(x(\hat{\l})) = (\lep- 1)\epsilon^*$ and so
  $\fs(x^t) \in ((\lep- 1)\epsilon^*, 1] = (\hat{y^t}, \hat{y}^t +
  \epsilon^*]$,
  so $\hat{y^t}$ is $\epsilon^*$-accurate. If neither $(a)$ nor $(b)$,
  then $(c)$ it must be $\hat{\l} < \lep -1$.  Since $|E^t| \leq 1$,
  for some $\l \in \{\hat{\l}+1, \hat{\l}+2\}$, we know that
  $p^{t,\l}_1 < p^{t,\l}_2$; thus,
  $\fs(x^t) < f(x(\l)) \leq (\hat{\l}+2)\epsilon^*$ and therefore
  $\fs(x^t) \in (\hat{\l}\epsilon^*, (\hat{\l}+2)\epsilon^*) =
  (\hat{y}^t, \hat{y}^t + 2\epsilon^*)$.\iffalse
  so $\hat{y}^t$ is $2\epsilon^*$ accurate.\fi

  Finally, we upper-bound $m(\epsilon, \delta)$, the number of rounds
  $t : \B(x^t) = \bot$. For each such $t$, $\A$ runs on a random draw
  of one of two contexts, one of whose arms' payoffs differ by at
  least $\epsilon^*$. Thus, \iffalse regardless of $\fs(x^t)$, \fi for
  one of those contexts, either $\fs(x^t) \geq f(x) - \epsilon^*$ or
  $\fs(x^t) \leq f(x) - \epsilon^*$. In either case, since
  $p^{t, \l}_1 = p^{t,\l}_2 = \frac{1}{2}$ for $x(\l) = x$, $\A$
  suffers expected regret at least $\frac{\epsilon^*}{2}$ for that
  context, and at least $\frac{\epsilon^*}{4}$ when faced with one
  chosen at random. Thus,
  $m(\epsilon, \delta) \cdot \frac{\epsilon^*}{4} = m(\epsilon,
  \delta) \cdot \frac{\epsilon}{8} < R(m(\epsilon, \delta), \delta^*)
  = R(m(\epsilon, \delta), \frac{\epsilon\delta}{2T})$,
  since $\A$'s regret is upper bounded by this quantity over
  $m(\epsilon, \delta)$ rounds (which is an upper bound on the number
  of rounds for which $\A$ is actually run and updated).
\end{proof} 

%% file: conjunction.tex
\subsection{An Exponential Separation Between Fair and Unfair Learning}\label{subsec:conjunction}

In this section, we exploit the other direction of the equivalence we
have proven between fair contextual bandit algorithms and KWIK
learning algorithms to give a simple contextual bandit problem for
which fairness imposes an \emph{exponential} cost in its regret
bound. This is in contrast to the case in which the underlying class
of functions is linear, for which we gave fair contextual bandit
algorithms with regret bounds within a polynomial factor of their
unconstrained counterparts.  In this problem, the context domain is
the $d$-dimensional boolean hypercube: $\cX = \{0,1\}^d$ -- i.e. the
context each round for each individual consists of $d$ boolean
attributes. Our class of functions $C$ is the class of boolean
\emph{conjunctions}:
$$C = \{f \mid f(x) = x_{i_1} \wedge x_{i_2} \wedge \ldots \wedge
x_{i_k} \text{ where } 0 \leq k \leq d \text{ and } i_1, \ldots, i_k
\in [d]\}.$$

We first give a simple but unfair algorithm,
\textsc{ConjunctionBandit}, for this problem which obtains a regret
bound which is linear in $d$. It maintains a set of candidate
variables $C_j^*$ for each conjunction $f_j$; this set shrinks across
rounds, while always containing the true set of variables over which
$f_j$ is defined. We denote the boolean value of variable $m$ in the
context for arm $j$ in round $t$ by $x_{j,m}^t$.

The formal claim and proof that \textsc{ConjunctionBandit} achieves
regret $R(T) = O(k^2d)$, as well as \textsc{ConjunctionBandit}'s
formal description, can be found in \app{sec:app-missing}.
\textsc{ConjunctionBandit} violates the fairness in every round $t$ in
which it predicts 0 for arm $i$ but 1 for arm $j$ even though
$f_i(x^t) = f_j(x^t) = 1$, as
$\pp{t}{i}{} = 0 < \frac{1}{k} < \pp{t}{j}{}$.

We now show that fair algorithms cannot guarantee subexponential
regret in $d$. This relies upon a known lower bound for KWIK learning
conjunctions ~\citep{li09}:

\begin{lemma} \label{lem:lilem} There exists a sequence of examples
  $(x^1, \ldots, x^{2^d - 1})$ such that for
  $\epsilon,\delta \leq 1/2$, every $(\epsilon,\delta)$-KWIK learning
  algorithm $\B$ for the class $C$ of conjunctions on $d$ variables
  must output $\bot$ for $x^t$ for each $t\in [2^d-1]$. Thus, $\B$ has
  a KWIK bound of at least $m(\epsilon,\delta) = \Omega(2^d)$.
\end{lemma}

We then use the equivalence between fair algorithms and KWIK learning
to translate this lower bound on $m(\epsilon,\delta)$ into a minimum
worst case regret bound for fair algorithms on conjunctions. We modify
Theorem \ref{thm:fairtokwik} to yield the following lemma, proven in
\app{sec:app-missing}.

\begin{lemma} \label{lem:fairtokwikconjunction} Suppose $\A$ is a
  $\delta$-fair algorithm for the contextual bandit problem over the
  class $C$ of conjunctions on $d$ variables. If $\mathcal{A}$ has
  regret bound $R(T, \delta)$ then for $\delta' = 2T\delta$, \fairkwik
  is an $(0, \delta')$-KWIK algorithm for $C$ with KWIK bound
  $m(0,\delta') = 4 R(m(0, \delta'),\delta)$.
\end{lemma}

Lemma \ref{lem:lilem} then lets us lower-bound the worst case regret
of fair learning algorithms on conjunctions.

\begin{corollary}
  For $\delta < \frac{1}{2T}$, any $\delta$-fair algorithm for
  the contextual bandit problem over the class $C$ of conjunctions on
  $d$ boolean variables has a worst case regret bound of
  $R(T) = \Omega(2^d)$.
\end{corollary}

\begin{proof}
  Let $T \leq 2^{d-1}$. We know then that if $\delta' < 1$,
  Lemma~\ref{lem:lilem} guarantees the existence of a sequence of
  contexts $x^1, \ldots, x^{T}$ for which any $(0,\delta')$-KWIK
  algorithm has KWIK bound $m(T, 0, \delta') = T$.

  Lemma~\ref{lem:fairtokwikconjunction} implies
  $4R(m(T, 0, \delta'), \delta)$ gives a KWIK bound of
  $m(T, 0, \delta')$ when $\delta' = 2T \delta$. Thus, if
  $\delta < \frac{1}{2T}$, then $\delta' < 1$ and so
  $R(m(T, 0, \delta'), \delta) = \frac{m(T, 0, \delta')}{4} =
  \frac{T}{4}$  . \end{proof}

Together with the analysis of \textsc{ConjunctionBandit}, this
demonstrates a strong separation between fair and unfair contextual
bandit algorithms: when the underlying functions mapping contexts to
payoffs are conjunctions on $d$ variables, there exist a sequence of
contexts on which fair algorithms must incur regret exponential in $d$
while unfair algorithms can achieve regret linear in $d$.

%% file: appendix.tex
\appendix
\section{Missing Proofs for the Classic Stochastic Bandits Upper
  Bound}\label{sec:app-noncontextual-regret}

We begin by proving Lemma \ref{lem:intervals}, used in Section
\ref{sec:fairbandits} to prove the fairness of the \fairbandits
algorithm.

\begin{proof}[Proof of Lemma~\ref{lem:intervals}]
  Choose an arbitrary arm $i$ and round $t$ and define indicator
  variables $X_1, \ldots, X_{n_i(t)}$ where $X_n$  takes on the reward
  of pull $n$ of arm $i$. By a Chernoff bound, for any $a \geq 0$,
\begin{equation*}
    \pr{|\mean{t}{i} - \mu_i| \leq  \frac{a}{\num{t}{i}}} \leq 2\exp(-2a^2/\num{t}{i}).
\end{equation*}
In particular for $a = \sqrt{\num{t}{i}\ln((\pi t)^2/3\delta)/2}$, it is the case that
\begin{align*}
    \; & \pr{\mu_i \not \in \left(\mean{t}{i}- \sqrt{\frac{\ln((\pi t)^2/3\delta)}{2\num{t}{i}}}, \mean{t}{i} + \sqrt{\frac{\ln((\pi t)^2/3\delta)}{2\num{t}{i}}}\right)}\\
 \leq & 2\exp(-2\num{t}{i}\ln((\pi t)^2/3\delta)/2\num{t}{i})
    =2\exp(\ln(3\delta/(\pi t)^2))
    = 6\delta/(\pi t)^2.
\end{align*}
By a union bound over all rounds $t$, the probability of any true mean
ever falling outside of its confidence interval is at most
$ \delta(\frac{6}{\pi^2}\sum_{t=1}^{\infty} \frac{1}{t^2}) = \delta$.
\end{proof}

Next, we prove Lemma \ref{lem:num-pulled}, which we used in Section \ref{sec:fairbandits} to bound the regret of \fairbandits in Theorem \ref{thm:noncontextual-regret}.

\begin{proof}[Proof of Lemma~\ref{lem:num-pulled}]
  Let $X_1, . . . , X_t$ be indicator variables of whether $i$ was
  pulled at each time $t'\in [t]$.  Let
  $M_t = \sum_{t' \leq t} X_{t'}$, with $\E{M_t} = p_t$. For any
  $\epsilon \in [0,1]$, a standard additive Chernoff bound states that
\[\pr{M_t  \leq p_t - \epsilon t} \leq e^{-2t\epsilon^2}.\]
Since $i\in S^t$, it must be that $i\in S^{t'}$ for all $t' \leq t$
and all $i\in S^t$, by the definition of \fairbandits. Thus,
$\pr{X_i = 1} \geq \frac{1}{k}$ for any $i\in S^t$, and therefore
$p_t \geq\frac{t}{k}$.  so this also implies that
\[\pr{M_t  \leq \frac{t}{k} - \epsilon t} \leq e^{-2t\epsilon^2}.\]
Setting $\epsilon t =\sqrt{\frac{t\ln(\frac{2t^2k}{\delta})}{2}}$, this bound becomes
\[\pr{M_t  \leq \frac{t}{k} - \sqrt{\frac{t\ln(\frac{2t^2k}{\delta})}{2}}}\leq \frac{\delta}{2kt^2}\]
as desired. Then, taking a union bound over all active arms of which
there are at most $k$, the claim follows.
\end{proof}

\begin{proof}[Proof of Lemma~\ref{lem:interval-width}]
  This follows from the definition of $\low{t}{i}, \up{t}{i}$ and the
  lower bound on $\num{t}{i}$ provided by the assumption of the lemma.
\end{proof}

\subsection{Missing Derivation of $R(T)$ for Theorem~\ref{thm:noncontextual-regret}}\label{sec:derivation}
 \begin{align*}
   R(T) &\leq \sum_{t: 0}^T \min(1, k \cdot \eta(t)) + \left(1 + \frac{\pi}{2}\right)\delta T \\
        &\leq \sum_{t: 0}^T k \cdot \min(1, \eta(t))+ \left(1 + \frac{\pi}{2}\right)\delta T \\
        & \leq  k \left(\sum_{t: \frac{t}{k} > 2 \sqrt{t\ln\frac{tk}{\delta}} }^T  \sqrt{\frac{\ln\frac{t}{\delta}}{\frac{t}{k} - \sqrt{t\ln\frac{tk}{\delta}}}} + \sum_{t: \frac{t}{k} \leq 2 \sqrt{t\ln\frac{tk}{\delta}} }^T 1  \right) + O(\delta) T\\
        & \leq  k \left(\sum_{t: \frac{t}{k} > 2 \sqrt{t\ln\frac{tk}{\delta}} }^T  \sqrt{\frac{\ln\frac{t}{\delta}}{\frac{t}{2k}}}+ \sum_{t: \frac{t}{k} \leq 2 \sqrt{t\ln\frac{tk}{\delta}} }^T  1  \right) + O(\delta T)\\
        & \leq  k \left(\int_{t= 0}^{T}   \sqrt{\frac{\ln\frac{t}{\delta}}{\frac{t}{2k}}} + \int_{t=1}^{\tilde{O}(k^2\ln\frac{k}{\delta})}  1\right) + O(\delta T)\\
        & \leq k^{\frac{3}{2}} \int_{t=1}^T \sqrt{\frac{\ln\frac{t}{\delta}}{t}} + \tilde{O}(k^3 \ln \frac{k}{\delta}) + O(\delta T)\\
        & = k^{\frac{3}{2}}\sqrt{2T}\sqrt{\ln \frac{kT}{\delta}} + \tilde{O}(k^3 \ln \frac{k}{\delta}) + O(\delta T)\\
        & = \tilde{O}(k^\frac{3}{2} \sqrt{T \ln \frac{kT}{\delta}} + k^3) + O(\delta T)\\
 & = \tilde{O}(k^\frac{3}{2} \sqrt{T \ln \frac{kT}{\delta}} + k^3) \\
\end{align*}

where the final step follows from $\delta \leq \frac{1}{\sqrt{T}}$.

\section{Missing Proofs for the Classic Stochastic Bandits Lower Bound}\label{sec:app-lower}

All lemmas in this section are used in Section \ref{sec:lower} to
prove the fair lower bound in Theorem \ref{thm:cubed-lower}. The
first, Lemma \ref{lem:bayes-exps}, lets us analyze distributions over
payoffs.

\begin{proof}[Proof of Lemma~\ref{lem:bayes-exps}]
  Let $R_i$ represent the joint distribution on rewards for either
  experiment: in both cases, the joint distribution on rewards is
  identical, since the process which generates them is the same.
  %The
  %different between the two distributions can only come from the joint
  %distribution between the rewards and the means, since the means
  %according to $W$ and to $W'$ are generated by different experiments.

  We will use the notation $m, d^1, \ldots, d^t$ to represent some
  fixed realization of the random variables
  $\mu_i, \rew{1}{i}, \ldots, \rew{t}{i}$ and
  $\mu'_i, \rew{1}{i}, \ldots, \rew{t}{i}$. In particular, it suffices
  to show that
\[\Pr{(\mu_i, \rew{1}{i}, \ldots, \rew{t}{i})  \sim W }{(\mu_i, \rew{1}{i}, \ldots, \rew{t}{i}) = (m, d^1, \ldots, d^t)} =  \Pr{(\mu'_i, \rew{1}{i}, \ldots, \rew{t}{i})  \sim W'}{(\mu'_i, \rew{1}{i}, \ldots, \rew{t}{i}) = (m, d^1, \ldots, d^t)}.\]
The first experiment which generates
$ (\mu_i, \rew{1}{i}, \ldots, \rew{t}{i})$ according to $W$ has
probability mass on this particular value of its random variables:
\begin{align*}
& \Pr{(\mu_i, \rew{1}{i}, \ldots, \rew{t}{i})  \sim W }{(\mu_i, \rew{1}{i}, \ldots, \rew{t}{i}) = (m, d^1, \ldots, d^t)}  = \Pr{\mu_i \sim P_i}{\mu_i = m} \cdot  \Pr{\rew{1}{i}, \ldots, \rew{t}{i} \sim B(\mu_i)}{(\rew{1}{i}, \ldots, \rew{t}{i}) = (d^1, \ldots, d^t)}
\end{align*}
%where the equality follows from Bayes rule.

The second experiment has joint probability:
\begin{align*}
& \Pr{(\mu'_i, \rew{1}{i}, \ldots, \rew{t}{i})  \sim W'}{(\mu'_i, \rew{1}{i}, \ldots, \rew{t}{i}) = (m, d^1, \ldots, d^t)}\\
  & = \Pr{\mu'_i \sim P_i(\rew{1}{i}, \ldots, \rew{t}{i})}{\mu'_i= m} \cdot  \Pr{(\rew{1}{i}, \ldots, \rew{1}{t})\sim R_i}{(\rew{1}{i}, \ldots, \rew{1}{t}) = (d^1, \ldots, d^t)}\\
 & = \Pr{(\mu'_i, \rew{1}{i}, \ldots, \rew{t}{i}) \sim W}{(\mu'_i, \rew{1}{i}, \ldots, \rew{t}{i}) = (m, d^1, \ldots, d^t)}
\end{align*}
where equality follows from Bayes' Rule.
\end{proof}

Next, we prove Lemma \ref{lem:distinguishing}, used to reason about distinguishing between arms.

\begin{proof}[Proof of Lemma~\ref{lem:distinguishing}]

  %We show that if no arm is $\sqrt{\delta}$-distinguished by $h^t$,
  %that $\pr{\mu_i = \mu_{i+1}|h^t} > \delta$.
  Since neither $i$ nor
  $i+1$ is $\sqrt{\delta}$-distinguished by $h^t$, for any
  $\alpha_i \in \{\frac{1}{3} + \frac{i}{3k},\frac{1}{3} + \frac{i+
    1}{3k} \}, \alpha_{i+1} \in \{\frac{1}{3} +
  \frac{i+1}{3k},\frac{1}{3} + \frac{i+ 2}{3k} \}$,
  the posterior probability of $\mu_i = \alpha_i$ is less than $1-\sqrt{\delta}$, and in particular for
  $\alpha =\frac{1}{3} + \frac{i+1}{3k}$, it must be the case that

\[ \Pr{\mu_i \sim P_i(h^t)}{\mu_i = \alpha} > \sqrt{\delta}\quad \textrm{ and also }\quad
 \Pr{\mu_{i+1}\sim P_i(h^t)}{\mu_{i+1} = \alpha } > \sqrt{\delta}. \]
Since $P = \prod_i P_i$, we know that
\[\Pr{\mu_i, \mu_{i+1} \sim P(h^t)}{\mu_{i} = \alpha = \mu_{i+1} } =
\Pr{\mu_i \sim P_i(h^t)}{\mu_i = \alpha} \cdot \Pr{\mu_{i+1} \sim P_{i+1}(h^t)}{\mu_{i+1} = \alpha} > \delta\]
which completes the proof.
\end{proof}

Finally, we prove Lemma \ref{lem:nonuniform}, which lets us reason about how fair algorithm choices depend on histories.

\begin{proof}[Proof of Lemma~\ref{lem:nonuniform}]
  We will define a set of histories which cause $\A$ to play some
  pair of arms $i$ and $i+1$ with different probabilities when
  $\mu_i = \mu_{i+1}$. Define the set $\unfair(\A,\mu)$ such that
  $h^t \in \unfair(\A, \mu)$ if there exist $i\in [k-1], t' \in [t]$
  such that $\mu_i = \mu_{i+1}$ but
  $\ppc{t'}{i}{h^{t'}} \neq \ppc{t'}{i+1}{h^{t'}}$.

  Consider some $h^t$ which has not $\sqrt{2\delta}$-distinguished any
  arm, such that there exists some $i, t'$ for which
  $\ppc{t'}{i}{h^{t'}} \neq \frac{1}{k}$. Then, in particular, there
  exists some $i\in [k-1]$ such that
  $\ppc{t'}{i}{h^{t'}} \neq \ppc{t'}{i+1}{h^{t'}}$. By
  Lemma~\ref{lem:distinguishing}, for all $i$ and in particular this
  $i$, it is the case that
  $2 \delta < \Pr{\mu' \sim P| h^t}{\mu'_i = \mu'_{i+1}} = X$ and so
  \begin{align} 2\delta < X = \Pr{\mu'
    \sim P| h^t}{\mu'_i = \mu'_{i+1} \cap \ppc{t'}{i}{h^t} \neq
    \ppc{t'}{i+1}{h^t} } \leq \Pr{\mu' \sim P| h^t}{h^t \in \unfair(\A, \mu')}\label{eqn:unfair}\end{align}
  where the first equality comes from the fact that $h^t$ is a history
  for which $ \ppc{t'}{i}{h^t} \neq \ppc{t'}{i+1}{h^t}$ and the second
  equality from the definition of the set unfair.

  We will show that Equation~\ref{eqn:unfair} cannot hold with
  probability more than $\frac{1}{2}$ over the draw of $\mu, h^t$ from
  the underlying distribution, or else $A$ would not satisfy
  $\delta$-fairness. Since $A$ is $\delta$-fair, for any fixed $\mu$
\[\delta  \geq \Pr{h^t \sim A | \mu}{h^t\in \unfair(\A, \mu)}\]
and therefore for any distribution $P$ over $\mu$ that
\[\delta  \geq \Pr{\mu \sim P, h^t \sim A | \mu}{h^t\in \unfair(\A, \mu)}.\]
Lemma~\ref{lem:bayes-exps} implies  also $\delta  \geq \Pr{\mu \sim P, h^t \sim A | \mu, \mu' \sim P|h^t}{h^t\in \unfair(\A, \mu')}$, so by Markov's inequality
\[\frac{1}{2}  \geq \Pr{\mu \sim P, h^t \sim A | \mu}{\Pr{ \mu' \sim P|h^t}{h^t\in \unfair(\A, \mu')}\geq 2\delta}.\]
Thus, with probability at least $\frac{1}{2}$ over the
distribution over histories and  means,
\[\Pr{ \mu' \sim P|h^t}{h^t\in \unfair(\A, \mu')}\leq 2\delta.\] However,
Equation~\ref{eqn:unfair} shows this does not hold for any $h^t$ which
does not $\sqrt{2\delta}$-distinguish any arm but for which
$\ppc{t'}{i}{h^{t'}} \neq \frac{1}{k}$ for some
$i \in [k], t' \leq t$.  Thus, for at least $\frac{1}{2}$ of all
probability mass over histories, either
$\ppc{t'}{i}{h^{t'}} = \frac{1}{k}$ for all $i, t' \leq t$, or $h^t$
must $\sqrt{2\delta}$-distinguish some arm.
\end{proof}

\section{Missing Proofs for the Contextual Bandit Setting}\label{sec:app-missing}

We begin by proving two results related to \kwikfair. The first, Lemma \ref{lem:prfailure}, was used in Section \ref{sec:kwiktofair} to prove that \kwikfair is $\delta$-fair in Theorem \ref{thm:kwiktofairfixed}.

\begin{proof}[Proof of Lemma \ref{lem:prfailure}]
  We will refer to a violation of either (a) or (b) as a failure of
  learner $L_i$.  For each $L_i$, the set of queries asked of it are
  pairs $(h_i, x^t_i)$, histories along with new contexts. There are
  at most $T$ contexts queried, and at most $T$ histories on which
  $L_i$ is queried for a fixed run of our algorithm (namely, prefixes
  of $L_i$'s final history). Thus, there are at most $T^2$ queries for
  $L_i$.  Thus, by a union bound over these $T^2$ queries for learner
  $L_i$, by the KWIK guarantee,
  $\Pr{}{L_i \text{ fails in some round}}$
  $\leq T^2 \delta^* = \min(\delta, \frac{1}{T})/k$, and by a union
  bound over $k$ arms,
  $\Pr{}{\text{A learner fails in a round}} \leq \min(\delta,
  \frac{1}{T})$.
\end{proof}

We proceed to Theorem \ref{thm:kwiktofair}, used in Section \ref{sec:kwiktofair} to construct a $\delta$-fair algorithm with quantified regret from KWIK learners.

\begin{proof}[Proof of Theorem \ref{thm:kwiktofair}]
  We use repeated calls to \kwikfair$(\delta,T)$ to run for an
  indefinite number of rounds. Specifically, we will make calls
  $E = 1, 2, \ldots$ to \kwikfair$(6\delta/\pi(\log(T)^2,2^E)$. We will
  refer to each such call to \kwikfair by its epoch $E$. By Lemma
  \ref{lem:prfailure}, each epoch $E$ is
  $6\delta/\pi E^2k$-fair, i.e. has a $6\delta/\pi E^2k$ probability
  of violating fairness. Therefore by a union bound across epochs, the
  probability of ever violating fairness through repeated calls to
  \kwikfair is bounded above by
  $\sum_{E=1}^{\infty} \frac{6\delta}{(\pi E)^2} =
  \frac{6\delta}{\pi^2} \sum_{E=1}^{\infty} \frac{1}{E^2} = \delta$,
  so the overall algorithm is $\delta$-fair.
	
  Next, by Lemma \ref{lem:kwiktofairfixedregret} each epoch $E$
  contributes at most regret $3 \cdot 2^Ek\epsilon^*_E$ where
  $\epsilon^*_E$ denotes the value of $\epsilon^*$ used in epoch $E$,
  i.e. $\epsilon^*_E$ satisfying
  $\epsilon^*_E = k \cdot m(\epsilon^*_E, 6\delta/\pi E^2,2^E)$. Then
  since each epoch $E$ covers $2^E$ rounds, through round $T$ the
  algorithm has used fewer than $\log(T)$ epochs, and by the doubling
  trick achieves regret
  $R(T) < \sum_{E=1}^{\log(T)} 3 \cdot 2^Ek\epsilon^*_E =
  O(Tk\epsilon^*) = O(k^2 \cdot m(\epsilon^*,\delta^*))$.
\end{proof}

Next, we address the special subcase of \kwikfair for linear functions
outlined in Corollary \ref{cor:kwiktofairlinear}.

\begin{proof}[Proof of Corollary~\ref{cor:kwiktofairlinear}]
  By Lemma \ref{lem:kwiktofairlinear}, for each arm $j$ the associated
  learner $L_j$ has mistake bound
  $m(\epsilon,\delta) = \tilde O(d^3/\epsilon^4)$. Since $\epsilon^*$
  satisfies $\epsilon^* = k \cdot m(\epsilon^*, \delta)/T$ we get
  $\epsilon^* = \left( \frac{kd^3}{T} \right)^{1/5}$ Substituting this
  into Theorem \ref{thm:kwiktofair}, the overall regret guarantee
  satisfies regret
  $R(T) = O(k^2 \cdot m(\epsilon^*,\delta^*)) = O(Tk\epsilon^*) =
  O(T^{4/5}k^{6/5}d^{3/5})$.
\end{proof}

This brings us to the formal algorithm description of \textsc{ConjunctionBandit} and its corresponding regret bound, used in Section \ref{subsec:conjunction} as an example of an unfair learning algorithm for conjunctions.

\begin{center}
	\begin{algorithmic}[1]
	\Procedure{\textsc{ConjunctionBandit}}{}
% 		\For{$j = 1, 2, \ldots, k$}
			\State Let $C_j^* \gets \{1, 2, \ldots, d\}$ for all $j\in [k]$ \Comment{Initialize set of candidate variables for $f_j$}
%		\EndFor
		\For{$t = 1, 2, \ldots$}
		\State $S^t \gets \emptyset$ \Comment{Initialize active set of arms}
			\For{$j = 1, 2, \ldots, k$}
				\If{$\wedge_{m \in C_j^*} x_{j,m}^t = 1$}
					\State $S^t \gets S^t \cup \{j\}$ \Comment{Add arm $j$ to active set}
				\EndIf
			\EndFor
			\If{$S^t = \emptyset$}
				\State Pull arm $j^* \gets (x \in_R [k])$  \Comment{Pull arm at random}
				\If{$r^t_{j^*} = 1$}
					\State $C_{j^*}^* \gets C_{j^*}^* \setminus \{m \mid x_{j,m}^t = 0\}$
				\EndIf
			\Else
				\State Pull arm $j^* \gets (x \in_R S^t)$  \Comment{Pull arm from active set at random}
			\EndIf
		\EndFor
 	\EndProcedure
 	\end{algorithmic}
\end{center}

We can now upper bound the regret achieved by \textsc{ConjunctionBandit}.

\begin{lemma}\label{lem:conjunctregret}
  \textsc{ConjunctionBandit} achieves regret $R(T) = O(k^2d)$.
\end{lemma}

\begin{proof}[Proof of Lemma~\ref{lem:conjunctregret}]
  First, we claim that for every $j$, for the duration of the
  algorithm, that $C_j \subseteq C_{j}^* $, where $C_j$ is the true
  set of variables corresponding to $f_j$. This holds at
  initialization: $C_j \subseteq [d] = C_{j}^*$.  Suppose the claim
  holds prior to round $t$: if $C^*_j$ is updated in this round, then
  $f_j(x^t_j) = 1 \Rightarrow \forall m\in [d] : x^t_{j,m} = 0,
  m\notin C_j$.
  Thus,
  $C^*_j = C^*_j \setminus \{m : x^t_{j,m} = 0\} = C^*_j \setminus \{m
  : x^t_{j,m} = 0 \cap m \notin C_j\} \supset C_j $.

  Therefore, the algorithm never makes false positive mistakes: in any
  round $t$, $j \in S^t \Rightarrow f_j(x_j^t) = 1$. Therefore
  \textsc{ConjunctionBandit} only accumulates regret in rounds where it makes
  false negative mistakes by predicting that all arms have reward 0
  when some arm has reward 1.

  Then, we have
  $\regret{x^1, \ldots, x^T} = \sum_{t}\max_j\left(f_j(x_j^t)\right) -
  \Ex{}{\sum_t f_{i^t}(x^t_{i^t})}$.
  We then rewrite the first term as
  $\sum_{t}\max_j\left(f_j(x_j^t)\right) = \sum_{t}
  \mathbb{I}\{f_j(x_j^t) = 1 \text{ for some } j \in [k]\}$
  and the second term as

\begin{align*}
	\Ex{}{\sum_t f_{i^t}(x^t_{i^t})} %=&\; \Ex{}{\sum_t \mathbb{I} \{S^t \neq \emptyset\} + \sum_t \mathbb{I} \{S^t = \emptyset \wedge f_{j^*}(x_{j^*}^t) = 1\}} \\
	=&\; T - \Ex{}{\sum_t \mathbb{I} \{S^t = \emptyset \wedge f_{j^*}(x_{j^*}^t) = 0 \wedge f_j(x_j^t) = 1 \text{ for some } j \in [k]\}} \\
	\geq&\; T - \Ex{}{\sum_t \mathbb{I} \{f_{j^*}(x_{j^*}^t) = 0 \mid S^t = \emptyset \wedge f_j(x_j^t) = 1 \text{ for some } j \in [k]\}} \\
	\geq&\; T - \sum_{j \in [k]} \Ex{}{\sum_t \mathbb{I} \{f_{j^*}(x_{j^*}^t) = 0 \mid S^t = \emptyset \wedge f_j(x_j^t) = 1\}} \\
	\geq&\; T - \sum_{j \in [k]} kd = T - k^2d
\end{align*}

where the last inequality follows from
$\pr{j^* = j \mid S^t = \emptyset \wedge f_j(x_j^t) = 1} =
\frac{1}{k}$
and the fact that if $S^t = \emptyset$ and $f_{j^*}(x_{j^*}^t) = 1$
then $C_{j^*}$ loses at least one of $d$ variables, and this loss can
therefore occur at most $d$ times for each arm $j$. Substituting this
into the original regret expression then yields

\begin{align*}
	\regret{x^1, \ldots, x^T} \leq&\; \sum_{t} \mathbb{I}\{f_j(x_j^t) = 1 \text{ for some } j \in [k]\} - (T - k^2d) \\
	\leq&\; T - (T - k^2d) \leq k^2d.
\end{align*}
\end{proof}

Finally, we prove Lemma~\ref{lem:fairtokwikconjunction}, which we used
in Section~\ref{subsec:conjunction} to translate between fair and KWIK
learning on conjunctions.

\begin{proof}[Proof of Lemma~\ref{lem:fairtokwikconjunction}]
  We mimic the structure of the proof of Theorem~\ref{thm:fairtokwik},
  once again using \fairkwik to construct a KWIK learner
  $\mathcal{B}$ by running the given fair algorithm $\mathcal{A}$ on a
  constructed bandit instance for each context $x^t$.
	
  There are two primary modifications for the specific case of
  conjunctions: as conjunctions output either 0 or 1 we set
  $\epsilon = 0$, $\epsilon^* = 1$, and
  $\delta^* = \frac{\delta}{2T}$. $\mathcal{A}$ therefore runs on $2T$
  histories and contexts, either of form $(x^t, x(0) = 0)$ or
  $(x^t, x(1) = 1)$. Since we initialize $\mathcal{A}$ to be
  $\delta^*$-fair, if we fix history $h^t$ along with context and arm
  assignment $(x^t, x(\l))$ then, with probability at least
  $1 - \delta^*$, $p_1^{t,\l} > p_2^{t,\l}$ implies $f^*(x^t) > \l/2$
  and similarly $p_2^{t,\l} > p_1^{t,\l}$ implies $f^*(x^t) <
  \l/2$.
  Union bounding over all such $t$ and $\l$ yields that $\mathcal{A}$
  satisfies this fairness over all $t$ and $\l$ with probability at
  least $1 - \delta$, and we condition on this event for the rest of
  the proof.
	
  We proceed to prove that the resulting KWIK learner $\mathcal{B}$ is
  $\epsilon$-accurate. Here, as $\epsilon = 0$, this requires showing
  that all of $\mathcal{B}$'s numerical predictions are
  correct. Assume instead that $\mathcal{B}$ outputs an incorrect
  prediction on $(x^t, x(\l))$. By the construction of \fairkwik, a
  prediction from $\mathcal{B}$ implies that at least one of
  $p_1^{t,0}, p_1^{t,1}, p_2^{t,0}$ and $p_2^{t,1}$ is distinct from
  the others. We condition on this distinctness to get two cases.

  In the first case, $p_1^{t,\l} \leq p_2^{t,\l}$ for both $\l = 0$
  and 1. By distinctness, this means that either
  $p_1^{t,0} < p_2^{t,0}$ or $p_1^{t,1} < p_2^{t,1}$. By the fairness
  assumption, this respectively implies that $f^*(x^t) < f(x(0)) = 0$
  or $f^*(x^t) < f(x(1)) = 1$. In either event,
  $f^*x(^t) = 0 = \hat y^t$. In the second case,
  $p_1^{t,\l} > p_2^{t,\l}$ for at least one of $\l = 0$ or
  1. $p_1^{t,1} > p_2^{t,1}$ violates the fairness assumption on
  $\mathcal{A}$ as $f(x(1)) = 1$, so it must be that
  $p_1^{t,0} > p_2^{t,0}$. Fairness then implies that
  $f^*(x^t) = 1 = \hat y^t$. Therefore $\mathcal{B}$ is
  $\epsilon$-accurate.

  It remains to upper bound $m(\epsilon,\delta)$. Any round where
  $\mathcal{B}$ outputs $\perp$ means a choice between two contexts,
  one of which has a difference of 1 between arms. It follows that
  choosing randomly between both arms and contexts incurs expected
  regret $1/4$. Therefore
  $\frac{m(\epsilon,\delta)}{4} < R(m(\epsilon,\delta), \delta^*, d) =
  R(m(\epsilon,\delta), \frac{\delta}{2T}, d)$.
\end{proof}